\documentclass[number,sort]{ReportTemplate}

\usepackage{bm}
\usepackage{amsthm}
\usepackage{amsfonts}
\usepackage{amssymb}
\usepackage{amsmath}
\usepackage{mathtools}
\usepackage{mathrsfs}
\usepackage[normal]{caption}
\usepackage{booktabs}
\usepackage{tabularx}
\usepackage{multirow}
\usepackage{makecell}
\usepackage{algorithm}
\usepackage{subcaption}
\usepackage{algorithmic}
\usepackage{float}
\usepackage{graphicx}
\usepackage{xcolor}
\usepackage{tikz}
\usepackage[hidelinks]{hyperref}
\mathtoolsset{showonlyrefs}
\definecolor{rank1blue}{RGB}{0, 0, 205}   %
\definecolor{rank2purple}{RGB}{128, 0, 128} %
\definecolor{braninbonet}{HTML}{22C5CC}
\definecolor{braninpgs}{HTML}{FFD119}
\definecolor{braningtg}{HTML}{F20DF2}
\definecolor{braninugtl}{HTML}{08E823}

\begin{document}

\begin{frontmatter}

\title{Rethinking Learnability in Offline Data-driven Optimization}

\author{Chao Qian}                  
\ead{qianc@lamda.nju.edu.cn}
\author{Chen-Guang Wang}               
\ead{221840189@smail.nju.edu.cn}
\author{Rong-Xi Tan}
\ead{tanrx@lamda.nju.edu.cn}
\author{Ke Xue}
\ead{xuek@lamda.nju.edu.cn}
\address{
State Key Laboratory of Novel Software Technology, Nanjing University, Nanjing 210023, China\\
School of Artificial Intelligence, Nanjing University, Nanjing 210023, China}

\begin{abstract}
Black-Box Optimization (BBO) has broad applications, while traditional algorithms such as evolutionary algorithms and Bayesian optimization face efficiency challenges as real-world BBO problems grow increasingly complex. Data-driven optimization has been the most popular paradigm to improve the efficiency of BBO, by learning from data. Offline data-driven optimization seeks high-quality solutions using only a fixed set of previous evaluations, attracting substantial attention because it requires no additional online evaluations. Many offline optimization methods have been proposed, but a fundamental question remains unanswered: what learnability is sufficient for offline optimization? Prior theoretical studies show that Probably Approximately Correct (PAC) learnability is insufficient, as the optimal region may remain poorly learned even when most regions are well learned. In this paper, we propose algorithm-dependent learnability, which requires accuracy only on the optimizer's trajectory. We prove that its value-query form is sufficient for representative discrete settings, including greedy and local search for submodular maximization, while its first-order analogue is sufficient for projected gradient descent on convex minimization. Motivated by this notion, we formalize a trajectory-learning framework comprising trajectory construction, trajectory modeling, and candidate generation, and analyze existing trajectory-based methods under it. We further propose Uncertainty-aware Gradient-guided Trajectory Learning (UGTL), which constructs locally coherent improvement trajectories reflecting plausible search paths, models them with conditional diffusion, and selects a diverse candidate set. Our experiments show that UGTL achieves the best average rank, 3.1/25, among 25 methods on Design-Bench tasks, and confirm that our trajectory construction plays a significant role in the improvement. 
\end{abstract}


\end{frontmatter}

\section{Introduction}

Black-Box Optimization (BBO) refers to optimizing objective functions where neither analytic expressions nor derivatives of the objective are available. It has various applications in science and engineering, e.g., drug discovery~\cite{terayama2021black}, material design~\cite{frazier2015bayesian}, biodiversity curve optimization~\cite{lu2026complex}, chip design~\cite{DBLP:conf/nips/Shi0L023,xue2026bboplace}, just to name a few. Popular BBO algorithms include evolutionary algorithms~\cite{zhou2019evolutionary} and Bayesian optimization~\cite{bo-book}, which iteratively improve the quality of solutions through nature-inspired operators or Gaussian-process-guided sampling, and only require the evaluation results of objective functions during their running. Though having found many successful applications, their efficiency faces great challenges, as real-world BBO problems are becoming increasingly complex, e.g., the number of variables to be optimized is very large and the objective evaluation may incur a very high computational cost. 

BBO algorithms traditionally rely on expert-derived heuristics, e.g., crossover/mutation operators for evolutionary algorithms and acquisition functions for Bayesian optimization. Recently, there has been a growing interest in improving the efficiency of BBO algorithms by learning from previously collected data or accumulated data during the running of algorithms. Many data-driven optimization methods have been developed, e.g., learning surrogate models of objective functions~\cite{wistuba2021fewshot,sam_at_icml23}, effective search spaces~\cite{DBLP:conf/nips/Song0H022,DBLP:conf/nips/Wang0LH024}, selection of proper algorithms~\cite{he2025train,he2026budget}, efficient parameter configurations of algorithms~\cite{xue2022multi,lu2026sequential}, advanced components of algorithms~\cite{hsieh2021reinforced,lange2023discoveringga}, and even an entire algorithm in an end-to-end fashion~\cite{chen2022towards,song2025reinforced}.

Among the data-driven optimization methods, offline data-driven optimization~\cite{kim2026offline,design-bench}, also referred to as optimization from samples~\cite{balkanski2022limitations}, has recently attracted substantial attention in the BBO community. Instead of interacting with the objective function online, it aims to infer promising candidate solutions solely from offline historical evaluations, which are collected from previous attempts, simulations, or deployments, and exist in many real-world scenarios. Thus, offline optimization can eliminate the need for additional evaluation budgets, which is attractive since acquiring new evaluations may be expensive, time-consuming, or even risky. Meanwhile, without the opportunity of correcting poor decisions through new evaluations, an offline optimizer must therefore extrapolate from partial observations while avoiding unsupported high-value predictions.

Existing offline optimization methods address this difficulty from three directions. Forward methods fit a surrogate model from the offline dataset and then optimize the learned surrogate, using conservatism, regularization, ensembles, invariant representations, ranking, or structural assumptions to control extrapolation~\cite{coms,roma,nemo,iom,ict,tri-mentoring,fgm,ltr,lyu2026learnability}. Inverse and conditional generative methods instead model promising designs directly, i.e., learn to generate candidate solutions conditioned on desirable objective values~\cite{mins,cbas,ddom,gabo,dynamo}. More recently, trajectory-based methods have learned sequences of improving designs through autoregressive models, policies, diffusion models, gradient matching, or probabilistic bridges~\cite{mashkaria2023generative,chemingui2024offline,yun2024guided,hoang2024learning,dao2025root}. 

For offline optimization, a fundamental question is: \emph{where must the learned object be accurate for optimization to succeed?} The famous Probably Approximately Correct (PAC) learnability~\cite{valiant1984theory} guarantees that with high probability, a learned surrogate can approximate the objective function well when given a polynomial number of observations. However, prior impossibility results show that there are problem classes (maximum coverage~\cite{balkanski2022limitations}, unconstrained submodular minimization~\cite{balkanski2017minimizing}, and convex minimization~\cite{balkanski2017sample}) that are PAC learnable and efficiently optimizable yet not optimizable from samples. Intuitively, PAC learnability guarantees average predictive accuracy, while the accuracy around the optimal region can still be bad, leading to that no reasonable approximation is achievable for offline optimization. We therefore ask a more targeted question: \emph{what type of learnability is sufficient for offline optimization?}

We answer this question with \emph{algorithm-dependent learnability}, which localizes the accuracy requirement to the information queried along an optimizer's trajectory. We formalize a value-query form (in Definition~\ref{def-algo-learnability}) for optimizers driven by function-value comparisons and a first-order analogue (in Definition~\ref{def-first-order-algo-learnability}) for optimizers driven by gradients. The value-query condition yields offline guarantees for size-constrained monotone approximately submodular maximization with greedy algorithms (as shown in Theorem~\ref{theo-submodular-constrained}) and for unconstrained non-monotone submodular maximization with local search algorithms (as shown in Theorem~\ref{theo-submodular-unconstrained}); the first-order condition yields an offline guarantee for convex minimization with projected gradient descent (as shown in Theorem~\ref{theo-convex-gradient}). The common message is that effective offline optimization may not require globally accurate value prediction; instead, it is sufficient to learn accurately along the trajectory that an optimizer will actually visit.

Inspired by the proposed algorithm-dependent learnability, we formalize a trajectory-learning framework for offline optimization, comprising three stages: trajectory construction, trajectory modeling, and candidate generation. The trajectory construction concentrates supervision on plausible search paths instead of asking one model to be uniformly accurate over the design space. We analyze existing trajectory-based methods under this framework, and find that their trajectory construction procedures do not follow the search principle of algorithms well. Thus, we propose a new trajectory-learning method by instantiating the framework with Uncertainty-aware Gradient-guided Trajectory Learning (UGTL), which combines uncertainty-aware gradient-guided trajectory construction with conditional diffusion and diversity-aware selection. Specifically, UGTL constructs locally coherent improvement trajectories from offline data, biases their transitions with an uncertainty-aware surrogate gradient, trains a conditional diffusion model over complete trajectories, and diversifies the final candidate set through clustering.

Our experiments show that compared with 24 existing offline optimization methods, UGTL ranks first on average ($3.1/25$) across five Design-Bench tasks~\citep{design-bench}, a popular benchmark for offline optimization. We also empirically verify the superiority of our trajectory construction through trajectory-quality diagnostics and controlled BBOB~\cite{bbob-functions} case studies. Furthermore, it can be transferred to other downstream models and training procedures: replacing the original constructor in BONET's Transformer pipeline~\citep{mashkaria2023generative}, PGS's offline-RL pipeline~\citep{chemingui2024offline}, and GTG's diffusion pipeline~\citep{yun2024guided} with our UGTL constructor can improve performance in almost all cases. We also verify the benefit of our candidate selection with clustering, and show that the performance of UGTL is not very sensitive to the main hyperparameters.

The rest of this paper is organized as follows. Section~\ref{sec:offline-optimization} introduces offline optimization and existing methods. Section~\ref{sec-learnability} reviews previous negative theoretical results of PAC learnability for offline optimization and introduces algorithm-dependent learnability followed by its positive theoretical results under three case studies. Section~\ref{sec:method} formalizes the trajectory-learning framework for offline optimization inspired by algorithm-dependent learnability, analyzes existing trajectory-based methods, and proposes a new one. Section~\ref{sec:experiments} presents the experimental results. Finally, Section~\ref{sec:conclusion} concludes
the paper.

\section{Offline Optimization}
\label{sec:offline-optimization}

Offline optimization~\cite{kim2026offline, design-bench}, also called optimization from samples~\cite{balkanski2022limitations}, considers the problem of optimizing an unknown objective function using only a static offline dataset. Let $\mathcal{X}$ denote the search space and $f:\mathcal{X}\rightarrow \mathbb{R}$ denote the objective function. In the maximization setting, the goal is to find
\begin{equation}
\begin{aligned}
\bm{x}^{*}\in \arg\max_{\bm{x}\in \mathcal{X}} f(\bm{x}).
\end{aligned}
\end{equation}
However, the optimizer cannot query $f$ at arbitrary solutions. Instead, it is given only an offline dataset
\begin{equation}
\begin{aligned}
D=\{(\bm{x}_i,y_i)\}_{i=1}^{N}, \quad y_i=f(\bm{x}_i),
\end{aligned}
\end{equation}
which is usually collected by previous experiments, simulations, or deployments. An offline optimization method takes $D$ as input and returns one or a small batch of candidate solutions for final evaluation. The minimization setting can be handled similarly by replacing maximization with minimization or by considering $-f$. The key difficulty is that the high-value region of $f$ may be poorly covered by $D$, and optimizing a learned model can drive the search into out-of-distribution regions where the model is unreliable.

A common family of methods is the \textit{forward} approach. These methods first fit a surrogate model $\tilde{f}$ on $D$, and then optimize $\tilde{f}$ by gradient ascent, evolutionary algorithms~\cite{zhou2019evolutionary}, Bayesian optimization~\cite{bo-book}, or other optimizers. Because a naively trained surrogate can overestimate values outside the data distribution, recent forward methods introduce conservative regularization, normalized maximum likelihood, robust adaptation, invariant representation learning, co-training or mentoring mechanisms, graphical structures, and ranking-based objectives to improve the alignment between model learning and final optimization performance~\cite{coms,nemo,roma,iom,ict,tri-mentoring,fgm,ltr}. Forward methods are simple and broadly applicable, but their success relies critically on whether the learned surrogate is reliable on the solutions visited during optimization.

Another family is the \textit{backward} or generative approach. Instead of first learning a scoring function and then optimizing it, these methods directly model the distribution of promising solutions, often by learning an inverse mapping from desired objective values to designs or by steering a generative model toward high-value regions. Representative examples include model inversion networks, conditioning-by-adaptive-sampling methods, diffusion-based inverse models, generative adversarial Bayesian optimization, and distribution-matching methods~\cite{mins,cbas,ddom,gabo,dynamo}. This paradigm can naturally generate diverse candidates and is particularly attractive in high-dimensional design spaces, but it still faces the challenge of extrapolating beyond the support of the offline data when conditioning on unseen high objective values.

A more recent family learns from \textit{optimization trajectories}. Rather than treating offline samples as independent design--score pairs, these methods synthesize or infer trajectories that mimic how an online optimizer would improve solutions. For example, BONET constructs monotone trajectories and trains an autoregressive Transformer~\cite{mashkaria2023generative}; PGS formulates offline optimization as policy-guided search~\cite{chemingui2024offline}; GTG trains a conditional diffusion model on locally improving trajectories~\cite{yun2024guided}; MATCH-OPT uses trajectory-induced gradient matching to regularize surrogate learning~\cite{hoang2024learning}; and ROOT learns a probabilistic bridge between low- and high-value distributions~\cite{dao2025root}. These methods are especially relevant to our study because they implicitly focus learning on the regions traversed by an optimizer, rather than requiring accurate function approximation over the whole search space.

\section{What Learnability is Sufficient for Offline Optimization?}\label{sec-learnability}

For offline optimization where the target function to be optimized is unknown but only sampled data is available, can we solve such problems? A formal framework named \textit{optimization from samples} was proposed to measure the performance of offline optimization theoretically~\cite{balkanski2017sample,balkanski2022limitations,balkanski2017minimizing}, which is presented in Definition~\ref{def-offline-optimizable}. Note that we have renamed the framework as \textit{offline optimizable} for the unification of terminology. We have also omitted constraints that the functions need to satisfy for clearness.

\begin{definition}[Offline Optimizable~\cite{balkanski2017sample,balkanski2022limitations,balkanski2017minimizing}]\label{def-offline-optimizable}
A class $\mathcal{F}$ of functions is $\alpha$-offline-optimizable over distribution $\mathcal{D}$ if $\forall f\in \mathcal{F}$ and $\delta \in (0,1)$, given  a polynomial number of samples $\{\bm{x}_i,f(\bm{x}_i)\}^m_{i=1}$ drawn i.i.d. from the distribution $\mathcal{D}$, there exists an algorithm that returns a solution $\tilde{\bm{x}}$ such that, for maximization under the multiplicative setting,
\begin{equation}
\begin{aligned}\label{eq-offline-opt-multi}
\mathrm{Pr}_{\bm{x}_1,\ldots,\bm{x}_m \sim \mathcal{D}}\left[f(\tilde{\bm{x}}) \geq \alpha \cdot \max_{\bm{x}} f(\bm{x}) \right]\geq 1-\delta,
\end{aligned}
\end{equation}
where $\alpha \in (0,1]$; and for minimization under the additive setting,  
\begin{equation}
\begin{aligned}\label{eq-offline-opt-add}
\mathrm{Pr}_{\bm{x}_1,\ldots,\bm{x}_m \sim \mathcal{D}}\left[f(\tilde{\bm{x}})-\min_{\bm{x}} f(\bm{x}) \leq \alpha\right]\geq 1-\delta,
\end{aligned}
\end{equation}
where $\alpha \geq 0$.
\end{definition}

According to Definition~\ref{def-offline-optimizable}, a class of functions is $\alpha$-offline-optimizable if we can construct an algorithm which, given polynomially many samples, outputs an $\alpha$-approximate solution with respect to the optimal function value, with high probability. Then, we face a fundamental question:\\ \centerline{\textit{What conditions can make a class of functions offline optimizable?}}\\
Such sufficient conditions may inspire the design of better offline optimization algorithms in practice. However, if a function is not optimizable or learnable, it is intuitively hard to achieve the offline optimizability. Thus, we are particularly interested in functions that are both learnable and optimizable. That is, \\ \centerline{\textit{When functions are optimizable and learnable, are they offline optimizable?}}\\
As the optimizability of functions can be defined by approximation straightforwardly, we focus on the characterization of learnability. That is, we reformulate the above question as\\\centerline{\textit{When functions are optimizable, what learnability can make them offline optimizable?}}

\subsection{PAC Learnability is Not Sufficient}\label{subsec-pac-insufficient}

The most famous notion of learnability is Probably Approximately Correct (PAC) learnability~\cite{valiant1984theory}, which implies that, given polynomial number of samples, it is possible to learn a function to approximately mimic the function where the samples are drawn. Unfortunately, recent works~\cite{balkanski2017sample,balkanski2022limitations,balkanski2017minimizing} have shown that PAC learnability is generally insufficient for offline optimization. Specifically, there are classes of functions (e.g., maximum coverage~\cite{balkanski2022limitations}, unconstrained submodular minimization~\cite{balkanski2017minimizing}, and convex minimization~\cite{balkanski2017sample}) that are both optimizable and PAC learnable, but for which no reasonable approximation for offline optimization is achievable.

Balkanski et al.~\cite{balkanski2022limitations} proved that for the problem of maximum coverage which is both optimizable and learnable, no constant factor approximation for offline optimization is achievable using polynomially many samples drawn from any distribution. As the maximum coverage problem is a special instance of monotone submodular function maximization with a size constraint, it is well known that maximum coverage is optimizable because the greedy algorithm achieves a $(1-1/e)$-approximation ratio~\cite{nemhauser1978analysis}. In terms of learnability, it has been proved~\cite{badanidiyuru2012sketching} that coverage functions are $(1-\epsilon)$-PMAC learnable over any distribution, where $\epsilon >0$ is any constant. Probably Mostly Approximately Correct (PMAC) learnability as presented in Definition~\ref{def-pmac} is a generalization of the standard notion of PAC learnability by considering the approximation between the true function value $f(\bm{x})$ and the predicted value $\tilde{f}(\bm{x})$, i.e., requiring that $\beta \cdot \tilde{f}(\bm{x})\leq f(\bm{x})\leq  \tilde{f}(\bm{x})$ where $\beta \in (0,1]$. For the standard PAC learnability, $\beta=1$. Thus, a function is PMAC learnable implies that most of the function values can be approximately learned very well with high probability. Theorem~\ref{theo-negative-coverage} shows that even a multiplicative approximation better than $2^{-\Omega(\sqrt{n})}$ cannot be achieved for offline optimization using polynomially many samples drawn from any distribution.

\begin{definition}[PMAC Learnability~\cite{balcan2011learning}]\label{def-pmac}
A class $\mathcal{F}$ of functions is PMAC learnable on distribution $\mathcal{D}$ if there exists a learning algorithm such that $\forall f \in \mathcal{F}$ and $\epsilon,\delta>0$, when running the learning algorithm on a polynomial number of samples $\{\bm{x}_i,f(\bm{x}_i)\}^m_{i=1}$ drawn i.i.d. from the distribution $\mathcal{D}$, the algorithm returns a function $\tilde{f}$: 
\begin{equation}
\begin{aligned}
\mathrm{Pr}_{\bm{x}_1,\ldots,\bm{x}_m \sim \mathcal{D}}\left[\mathrm{Pr}_{\bm{x} \sim \mathcal{D}}\left[ \beta \cdot \tilde{f}(\bm{x})\leq f(\bm{x})\leq  \tilde{f}(\bm{x})\right]\geq 1-\epsilon\right]\geq 1-\delta,
\end{aligned}
\end{equation}
where $\beta \in (0,1]$, and $m$ is polynomial in $1/\epsilon$, $1/\delta$, and the size of $f \in \mathcal{F}$ (i.e., the dimensionality of $\bm{x} \in \mathcal{D}$).
\end{definition}

\begin{theorem}[Theorem 2.2 of~\cite{balkanski2022limitations}]\label{theo-negative-coverage}
For the problem of maximum coverage under a size constraint that is optimizable with an approximation ratio of $1-1/e$ and PMAC learnable, no algorithm can obtain a multiplicative approximation better than $2^{-\Omega(\sqrt{n})}$, using polynomially many samples drawn from any distribution.
\end{theorem}

Balkanski and Singer~\cite{balkanski2017minimizing} proved that even learnable submodular functions cannot be minimized within any non-trivial approximation when given access to polynomially many samples. Submodular functions can be minimized in polynomial time~\cite{grotschel2012geometric}. Balkanski and Singer~\cite{balkanski2017minimizing} constructed a family of submodular functions which is PAC learnable with absolute loss, where the sample complexity $m \in O(n^3+n^2(\log(2n/\delta))/\epsilon^2)$ with $n$ being the size of $f$. As presented in Definition~\ref{def-pac}, a function is PAC learnable with absolute loss implies that the expectation of the absolute loss between the true function value $f(\bm{x})$ and the predicted value $\tilde{f}(\bm{x})$, i.e., $\mathbb{E}_{\bm{x} \in \mathcal{D}}\left[\left|\tilde{f}(\bm{x})-f(\bm{x})\right|\right]$, can be well bounded with high probability. But as shown in Theorem~\ref{theo-negative-minsub}, for minimizing this family of submodular functions, despite being minimizable in polynomial time and PAC learnable, no algorithm can obtain an additive approximation better than $1/2-o(1)$ using polynomially many samples drawn from any distribution.

\begin{definition}[PAC Learnability with Absolute Loss~\cite{valiant1984theory}]\label{def-pac}
A class $\mathcal{F}$ of functions is PAC learnable on distribution $\mathcal{D}$ if there exists a learning algorithm such that $\forall f \in \mathcal{F}$ and $\epsilon,\delta>0$, when running the learning algorithm on a polynomial number of samples $\{\bm{x}_i,f(\bm{x}_i)\}^m_{i=1}$ drawn i.i.d. from the distribution $\mathcal{D}$, the algorithm returns a function $\tilde{f}$: 
\begin{equation}
\begin{aligned}
\mathrm{Pr}_{\bm{x}_1,\ldots,\bm{x}_m \sim \mathcal{D}}\left[\mathbb{E}_{\bm{x} \in \mathcal{D}}\left[\left|\tilde{f}(\bm{x})-f(\bm{x})\right|\right]\leq \epsilon\right]\geq 1-\delta,
\end{aligned}
\end{equation}
where $m$ is polynomial in $1/\epsilon$, $1/\delta$, and the size of $f \in \mathcal{F}$ (i.e., the dimensionality of $\bm{x} \in \mathcal{D}$).
\end{definition}

\begin{theorem}[Theorem~8 of~\cite{balkanski2017minimizing}]\label{theo-negative-minsub}
There exists a family of $[0,1]$-bounded submodular functions that can be minimized in polynomial time and is PAC learnable, but no algorithm can obtain an additive approximation better than $1/2-o(1)$ for minimizing this family of submodular functions without constraints, using polynomially many samples drawn from any distribution.  
\end{theorem}

Balkanski and Singer~\cite{balkanski2017sample} also proved that in general, the number of samples required to obtain a non-trivial approximation to the minimum of a convex function is exponential, even when the function is PAC learnable. It is known that convex functions can be efficiently optimized within an arbitrary degree of precision~\cite{nesterov2013introductory}. Balkanski and Singer~\cite{balkanski2017sample} constructed a class of convex functions which is PAC learnable with the sample complexity $m \in O(\epsilon^{-4}\delta^{-1}(n+\log(\delta^{-1})))$. Theorem~\ref{theo-negative-convex} shows that no algorithm can obtain an additive approximation better than $1/2-o(1)$ for minimizing this class of convex functions, using polynomially many samples drawn from any distribution. Furthermore, they proved that offline optimizability is impossible even assuming strong convexity and Lipschitz continuity of convex functions.

\begin{theorem}[Theorem~8 of~\cite{balkanski2017sample}]\label{theo-negative-convex}
There exists a class of convex functions $\mathcal{F}$ (where $f: [0,1]^n \rightarrow [0,1]$ for $f \in \mathcal{F}$) that can be efficiently minimized in polynomial time and is PAC learnable, but no algorithm can obtain an additive approximation better than $1/2-o(1)$ for minimizing this class of convex functions, using polynomially many samples drawn from any distribution.  
\end{theorem}

The reason why the above functions are not offline optimizable using polynomially many samples drawn from any distribution is similar. The studied class of functions is defined based on partitions of the set of all dimensions $\{1,2,\ldots,n\}$, and their optima crucially depend on the partitions. The hardness arises from the fact that some critical parts of the partitions (and thus optimal solutions) cannot be distinguished from polynomially many samples with high probability, while most parts of the partitions can be learned well (and thus the value of almost all solutions can be estimated well), which explains why these functions are learnable but not offline optimizable.

\subsection{Algorithm-dependent Learnability}

The negative results introduced in Section~\ref{subsec-pac-insufficient} show that no reasonable approximation is guaranteed when optimizing a function that has been learned from data, even if this function is both optimizable and PAC/PMAC learnable. That is, PAC/PMAC learnability is insufficient for offline optimization, i.e., to obtain reasonable approximation guarantees when optimizing a function that is learned from data. Thus, an interesting question is what learnability can guarantee reasonable approximations for offline optimization.

There have been some attempts in this direction by making strong assumptions or relaxing the approximation requirement~\cite{balkanski2022limitations,chen2020optimization,rosenfeld2018learning}. Balkanski et al.~\cite{balkanski2022limitations} proposed the strong concept of \textit{recoverability}, which requires a function to be learned everywhere within an approximation of $1\pm 1/n^2$ from samples, i.e., the learned function $\tilde{f}$ satisfies that for every solution $\bm{x}$, $(1-1/n^2)\cdot f(\bm{x}) \leq \tilde{f}(\bm{x}) \leq (1+1/n^2)\cdot f(\bm{x})$ with high probability. Under recoverability, the problem of monotone submodular maximization with a size constraint (and also the specific instance, maximum coverage) is $(1-1/e-o(1))$-offline-optimizable. However, the condition of recoverability is too strong to hold in practice, which is also shown not necessary in~\cite{balkanski2022limitations} since unit demand functions are offline optimizable but not recoverable. 

Chen et al.~\cite{chen2020optimization} circumvented the impossibility result of offline optimization for maximum coverage by proposing a stronger model called \textit{optimization from structured samples}, where the data samples encode the structural information of the functions, i.e., the samples could reveal the covered elements rather than just the number of covered elements. Under certain assumptions on the sample distribution, they designed an algorithm that achieves a constant approximation for offline optimization of the maximum coverage problem. However, the applicability of this strong model is limited because often only the values of the functions (rather than the structural information of the functions) are available in practice.

To circumvent the difficulty of offline optimization, Rosenfeld et al.~\cite{rosenfeld2018learning} proposed a new model called \textit{distributional optimization from samples} by replacing the approximation of the whole search space with that of a sampled subspace and using a distribution-agnostic notion of approximation. The authors established a tight equivalence between this model and PMAC learnability: A function class is distributional optimization from samples if and only if it is PMAC learnable. However, the approximation concerned in practice is usually with respect to the whole search space rather than a sampled subspace.

In this section, we propose a new learning concept called \textit{algorithm-dependent learnability} that can provide approximation guarantees when optimizing learned functions. As presented in Definition~\ref{def-algo-learnability}, given an algorithm $\mathcal{A}$ for solving a class $\mathcal{F}$ of functions, it requires that with high probability, the functions can be learned approximately well (i.e., the predicted function value $\tilde{f}(\bm{x})$ is bounded within $(1\pm \beta)$ of the true function value $f(\bm{x})$) along the trajectories of the algorithm $\mathcal{A}$ running on the learned functions. Note that a solution visited by $\mathcal{A}$ on the learned function $\tilde{f}$ in Eq.~(\refeq{eq-algo-learnability}) means that the solution is evaluated during the process of running $\mathcal{A}$ on $\tilde{f}$. Compared to recoverability~\cite{balkanski2022limitations}, the proposed algorithm-dependent learnability does not require that functions are learnable everywhere from samples.

\begin{definition}[Algorithm-dependent Learnability]\label{def-algo-learnability}
Let $\mathcal{A}$ be an algorithm for solving a class $\mathcal{F}$ of functions. Then, $\mathcal{F}$ is algorithm-dependent learnable with respect to $\mathcal{A}$ on distribution $\mathcal{D}$ if there exists a learning algorithm such that $\forall f \in \mathcal{F}$ and $\delta>0$, when running the learning algorithm on a polynomial number of samples $\{\bm{x}_i,f(\bm{x}_i)\}^m_{i=1}$ drawn i.i.d. from the distribution $\mathcal{D}$, the algorithm returns a function $\tilde{f}$: 
\begin{equation}
\begin{aligned}\label{eq-algo-learnability}
\mathrm{Pr}_{\bm{x}_1,\ldots,\bm{x}_m \sim \mathcal{D}}\left[ \forall \bm{x} \; \text{visited by running} \; \mathcal{A} \;\text{on} \; \tilde{f}: (1-\beta) \cdot f(\bm{x})\leq \tilde{f}(\bm{x})\leq  (1+\beta)\cdot f(\bm{x}) \right]\geq 1-\delta,
\end{aligned}
\end{equation}
where $\beta \in [0,1)$, and $m$ is polynomial in $1/\delta$ and the size of $f \in \mathcal{F}$ (i.e., the dimensionality of $\bm{x} \in \mathcal{D}$).
\end{definition}

Definition~\ref{def-algo-learnability} is tailored to algorithms whose decisions depend on function-value comparisons. For a first-order optimizer, the queried gradients, rather than the absolute function values, determine the trajectory. We therefore use the following first-order analogue, which controls gradient accuracy only at the points where the learned function is queried by the optimizer.

\begin{definition}[First-order Algorithm-dependent Learnability]\label{def-first-order-algo-learnability}
Let $\mathcal{A}$ be a first-order algorithm whose trajectory is determined by gradient queries for solving a class $\mathcal{F}$ of functions that are differentiable on an open neighborhood of a domain $\mathcal{X}$. Then, $\mathcal{F}$ is first-order algorithm-dependent learnable with respect to $\mathcal{A}$ on distribution $\mathcal{D}$ if there exists a learning algorithm such that $\forall f \in \mathcal{F}$ and $\delta>0$, when running the learning algorithm on a polynomial number of samples $\{\bm{x}_i,f(\bm{x}_i)\}^m_{i=1}$ drawn i.i.d. from the distribution $\mathcal{D}$, the algorithm returns a function $\tilde{f}$ that is differentiable on the open neighborhood of $\mathcal{X}$, and satisfies 
\begin{equation}
\begin{aligned}\label{eq-first-order-algo-learnability}
\mathrm{Pr}_{\bm{x}_1,\ldots,\bm{x}_m \sim \mathcal{D}}\left[
\forall \bm{x}\;\text{at which $\mathcal{A}$ queries a gradient when solving $\tilde{f}$}:\;
\|\nabla\tilde{f}(\bm{x})-\nabla f(\bm{x})\|\leq\zeta_g
\right]\geq 1-\delta,
\end{aligned}
\end{equation}
where $\zeta_g\geq0$, and $m$ is polynomial in $1/\delta$ and the size of $f\in\mathcal{F}$.
\end{definition}

Next, we show that the value-query condition in Definition~\ref{def-algo-learnability} is sufficient for two discrete settings: monotone approximately submodular maximization with a size constraint by a greedy algorithm (Section~\ref{subsec-submodular-greedy}) and unconstrained non-monotone submodular maximization by local search (Section~\ref{subsec-submodular-local-search}). For convex minimization by projected gradient descent (Section~\ref{subsec-convex-gradient}), the first-order condition in Definition~\ref{def-first-order-algo-learnability} is sufficient. By running an existing optimizer on the learned function \(\tilde f\), algorithm-dependent learnability can transfer the optimizer’s query-wise decision relations from \(\tilde f\) to the unknown objective \(f\), then the structural properties of \(f\), such as approximate submodularity or convexity, convert these local relations into a global approximation guarantee. The three cases instantiate the same principle at the level of the information consumed by each optimizer: reliable values for value-query algorithms and reliable gradients for first-order algorithms.

\subsubsection{Sufficient for Offline Size-constrained Monotone Approximately Submodular Maximization with Greedy Algorithms}\label{subsec-submodular-greedy}

First, we study the problem of monotone approximately submodular maximization with a size constraint, as presented in Definition~\ref{def-submodular-constraints}. Given a finite set $V=\{v_1,v_2,\ldots,v_n\}$, we consider the functions $f:2^V \rightarrow \mathbb{R}$ defined on subsets of $V$, where $\mathbb{R}$ denotes the set of reals. A set function $f:2^V \rightarrow \mathbb{R}$ is monotone if $\forall X \subseteq Y$, $f(X) \leq f(Y)$. We assume without loss of generality that monotone functions are normalized, i.e., $f(\emptyset)=0$; thus, we have $\forall X \subseteq V$, $f(X)\geq 0$. A set function $f$ is submodular~\cite{nemhauser1978analysis} if it satisfies the ``diminishing returns" property, i.e., $\forall X \subseteq Y \subseteq V$ and $v \notin Y$,
\begin{align}\label{def-submodular-1}
f(X \cup \{v\})-f(X) \geq f(Y \cup \{v\}) - f(Y);
\end{align}
or equivalently $\forall X \subseteq Y \subseteq V$,
\begin{align}\label{def-submodular-2}
f(Y)-f(X) \leq \sum\nolimits_{v \in Y \setminus X} \big(f(X \cup \{v\})-f(X)\big),
\end{align}
which implies that the increment in $f$ by adding a set of elements to a set $X$ is not larger than the sum of increments of adding its individual elements to $X$. For a monotone but not necessarily submodular function $f$, the submodularity ratio as presented in Definition~\ref{def-approx-submodular} was introduced to measure the closeness of $f$ to submodularity, i.e., to what extent $f$ has the submodular property. We can see that the submodularity ratio is defined based on Eq.~(\refeq{def-submodular-2}); $\gamma_{X,l}(f)\in [0,1]$, and $f$ is submodular if and only if $\gamma_{X,l}(f) = 1$ for any $X$ and $l$. In this paper, we will use $\gamma_{X,l}$ for short when the meaning of $f$ is clear. 

\begin{definition}[Submodularity Ratio~\cite{das2011submodular,qian.nips15}]\label{def-approx-submodular}
Let $f: 2^V \rightarrow \mathbb{R}$ be a set function. The submodularity ratio of $f$ with respect to a set $X \subseteq V$ and a parameter $l \geq 1$ is
$$
\gamma_{X,l}(f)=\min_{L \subseteq X, S: |S|\leq l, S \cap L =\emptyset} \frac{\sum_{v \in S} (f(L \cup \{v\})-f(L))}{f(L \cup S)-f(L)}.
$$
\end{definition}

As presented in Definition~\ref{def-submodular-constraints}, the studied problem is to find a subset $X \subseteq V$ maximizing a monotone approximately submodular function $f$ with a size constraint $|X| \leq k$. It is generally NP-hard, and has many applications, e.g., maximum coverage~\cite{feige1998threshold}, influence maximization~\cite{kempe2003maximizing}, sparse regression~\cite{das2011submodular}, unsupervised feature selection~\cite{feng2019unsupervised}, and human assisted learning~\cite{liu2023human}. Note that the objective functions of the first two applications are exactly submodular (i.e., $\forall X, l: \gamma_{X,l}(f) = 1$), while those of the last three ones are approximately submodular.

\begin{definition}[Monotone Approximately Submodular Function Maximization with a Size Constraint]\label{def-submodular-constraints}
Given a monotone and approximately submodular function $f: 2^V \rightarrow \mathbb{R}^+$ (where $\mathbb{R}^{+}$ denotes the set of non-negative reals) and a budget $k$, the goal is to find a subset $X^*\subseteq V$ such that
$$
X^*\in \arg \max\nolimits_{X \subseteq V, |X|\le k} f(X).
$$
\end{definition}

It has been shown that the greedy algorithm (presented in Algorithm~\ref{alg:Greedy}) achieves an approximation ratio of $1-e^{-\gamma_{X_k,k}}$, i.e., finds a subset $X_k$ with $f(X_k) \geq (1-e^{-\gamma_{X_k,k}})\cdot \max_{X \subseteq V: |X| \leq k}f(X)$~\cite{das2011submodular}. Furthermore, this approximation ratio has been proved to be optimal~\cite{harshaw2019submodular}. The process of the greedy algorithm is straightforward, which iteratively selects one element with the largest marginal gain on $f$ (i.e., $v^*\in\arg\max_{v \in V \setminus X_i} f(X_i \cup \{v\})$ in lines~3--4 of Algorithm~\ref{alg:Greedy}) until $k$ elements are selected.

\begin{algorithm}[htbp]
\caption{Greedy Algorithm~\cite{das2011submodular}}
\label{alg:Greedy}
\textbf{Input}: all elements $V=\{v_1,\ldots,v_n\}$, objective function $f$, and budget $k$\\
\textbf{Output}: a subset of $V$ with $k$ elements\\
\textbf{Process}:
\begin{algorithmic}[1]
\STATE Let $i=0$ and $X_i=\emptyset$;
\STATE \textbf{repeat}
\STATE \quad Let $v^*\in\arg\max_{v \in V \setminus X_i} f(X_i \cup \{v\})$;
\STATE \quad Let $X_{i+1}=X_{i} \cup \{v^*\}$, and $i=i+1$
\STATE \textbf{until} $i=k$
\STATE \textbf{return} $X_k$
\end{algorithmic}
\end{algorithm}

Theorem~\ref{theo-submodular-constrained} shows that the problem of monotone approximately submodular maximization with a size constraint is $\alpha$-offline-optimizable (Definition~\ref{def-offline-optimizable}) when the algorithm-dependent learnablity (Definition~\ref{def-algo-learnability}) with respect to the greedy algorithm holds. Note that a solution in Definitions~\ref{def-offline-optimizable} and~\ref{def-algo-learnability} is represented by a vector $\bm{x}$ whereas here it is represented by a subset $X$ of $V$; there is no inconsistency because a subset $X$ of $V$ can be naturally represented by a Boolean vector $\bm{x} \in \{0,1\}^n$, where the $i$-th bit $x_i=1$ means that $v_i \in X$, and $x_i=0$ means that $v_i \notin X$. The approximation ratio $\alpha=\frac{\frac{1-\beta}{1+\beta}\frac{\gamma_{X_k,k}}{k}}{1-\frac{1-\beta}{1+\beta}\left(1-\frac{\gamma_{X_k,k}}{k}\right)}
\left(1-\left(\frac{1-\beta}{1+\beta}\right)^{k}\left(1-\frac{\gamma_{X_k,k}}{k}\right)^k\right)$. When the objective function $f$ is exactly submodular, $\gamma_{X_k,k}=1$, and thus the approximation ratio $\alpha$ is specialized to $\frac{\frac{1-\beta}{1+\beta}\frac{1}{k}}{1-\frac{1-\beta}{1+\beta}\left(1-\frac{1}{k}\right)}
\left(1-\left(\frac{1-\beta}{1+\beta}\right)^{k}\left(1-\frac{1}{k}\right)^k\right)$. When $\beta$ is further set to $1/n^2$ as in recoverability~\cite{balkanski2022limitations}, $\alpha$ becomes $1-1/e-o(1)$, which is consistent with the result for the specific case, maximum coverage, under the strong learnability concept of recoverability (which requires a function to be learned everywhere within an approximation of $1\pm 1/n^2$ from samples)~\cite{balkanski2022limitations}. When $\beta=0$, the approximation ratio $\alpha$ recovers that, $1-(1-\frac{\gamma_{X_k,k}}{k})^k$, achieved by the greedy algorithm under the value oracle model where each solution can be exactly evaluated~\cite{das2011submodular}.  

\begin{theorem}\label{theo-submodular-constrained}
If the problem of monotone approximately submodular maximization with a size constraint is algorithm-dependent learnable with respect to the greedy algorithm on distribution $\mathcal{D}$, then it is $\alpha$-offline-optimizable over $\mathcal{D}$, where
$$
\alpha=\frac{\frac{1-\beta}{1+\beta}\frac{\gamma_{X_k,k}}{k}}{1-\frac{1-\beta}{1+\beta}\left(1-\frac{\gamma_{X_k,k}}{k}\right)}
\left(1-\left(\frac{1-\beta}{1+\beta}\right)^{k}\left(1-\frac{\gamma_{X_k,k}}{k}\right)^k\right).
$$
That is, Eq.~(\refeq{eq-offline-opt-multi}) holds with this $\alpha$.
\end{theorem}
\begin{proof}
By the algorithm-dependent learnability with respect to the greedy algorithm, we know from Definition~\ref{def-algo-learnability} that there is a learning algorithm which can return a function $\tilde{f}$ with probability $1-\delta$: 
\begin{equation}
\begin{aligned}\label{proof-eq-algo-learnability}
 \forall X \; \text{visited by running the greedy algorithm on} \; \tilde{f}: (1-\beta) \cdot f(X)\leq \tilde{f}(X)\leq  (1+\beta)\cdot f(X),
\end{aligned}
\end{equation}
when getting access to a polynomial number of samples $\{\bm{x}_i,f(\bm{x}_i)\}^m_{i=1}$ drawn i.i.d. from the distribution $\mathcal{D}$. In Eq.~(\refeq{proof-eq-algo-learnability}), $\beta \in [0,1)$, and $f$ denotes the true function.

We consider the greedy algorithm running on the learned function $\tilde{f}$ that satisfies Eq.~(\refeq{proof-eq-algo-learnability}). Let $X^*$ denote an optimal subset, i.e., $f(X^*)= \max_{X \subseteq V: |X| \leq k}f(X)$, and $X_i$ denote the subset after selecting $i$ elements by the greedy algorithm on $\tilde{f}$. We have
\begin{equation}
\begin{aligned}\label{inter-step}
&f(X^*)-f(X_i) \\
&\leq f(X^* \cup X_i)-f(X_i)& [\text{by the monotonicity of $f$}]\\
&\leq \frac{1}{\gamma_{X_i,k}}\sum_{v \in X^*\setminus X_i} \big(f(X_i \cup \{v\})-f(X_i)\big)& [\text{by Definition~\ref{def-approx-submodular} and $|X^*| \leq k$}]\\
&\leq \frac{1}{\gamma_{X_i,k}}\sum_{v \in X^*\setminus X_i} \left(\frac{1}{1-\beta}\cdot \tilde{f}(X_i \cup \{v\})-f(X_i)\right) & [\text{by Eq.~(\refeq{proof-eq-algo-learnability}) as $X_i \cup \{v\}$ is visited by Algorithm~\ref{alg:Greedy}}]\\
&\leq \frac{1}{\gamma_{X_i,k}}\sum_{v \in X^*\setminus X_i} \left(\frac{1}{1-\beta}\cdot \tilde{f}(X_{i+1})-f(X_i)\right) & [\text{by line~3 of Algorithm~\ref{alg:Greedy}}]\\
&\leq \frac{k}{\gamma_{X_k,k}} \left(\frac{1+\beta}{1-\beta}\cdot f(X_{i+1})-f(X_i)\right) & [\text{by $\gamma_{X_i,k} \geq \gamma_{X_{i+1},k}$ and Eq.~(\refeq{proof-eq-algo-learnability})}]
\end{aligned}
\end{equation}
The above equation can be equivalently transformed as
$$
f(X_{i+1}) \geq \left(\frac{1-\beta}{1+\beta}\right)\left(\left(1-\frac{\gamma_{X_k,k}}{k}\right)f(X_i)+\frac{\gamma_{X_k,k}}{k} f(X^*)\right).
$$
Then, by induction, the returned subset $X_k$ by Algorithm~\ref{alg:Greedy} satisfies
$$
f(X_k) \geq \frac{\frac{1-\beta}{1+\beta}\frac{\gamma_{X_k,k}}{k}}{1-\frac{1-\beta}{1+\beta}\left(1-\frac{\gamma_{X_k,k}}{k}\right)}
\left(1-\left(\frac{1-\beta}{1+\beta}\right)^{k}\left(1-\frac{\gamma_{X_k,k}}{k}\right)^k\right) \cdot f(X^*),$$
i.e., Eq.~(\refeq{eq-offline-opt-multi}) in Definition~\ref{def-offline-optimizable} holds with $$\alpha=\frac{\frac{1-\beta}{1+\beta}\frac{\gamma_{X_k,k}}{k}}{1-\frac{1-\beta}{1+\beta}\left(1-\frac{\gamma_{X_k,k}}{k}\right)}
\left(1-\left(\frac{1-\beta}{1+\beta}\right)^{k}\left(1-\frac{\gamma_{X_k,k}}{k}\right)^k\right).$$
\end{proof}

The proof of Theorem~\ref{theo-submodular-constrained} directly follows the proof for the approximation ratio of the greedy algorithm under multiplicative noise~\cite{qian2017subset}, because the condition $(1-\epsilon)\cdot f(X) \leq F(X) \leq (1+\epsilon) \cdot f(X)$ satisfied by the noisy objective function $F(X)$ is similar to Eq.~(\ref{proof-eq-algo-learnability}) satisfied by the learned function $\tilde{f}(X)$ under algorithm-dependent learnability. We can find from the upper bound analysis of $f(X^*)-f(X_i)$ (i.e., Eq.~(\refeq{inter-step})) in the proof of Theorem~\ref{theo-submodular-constrained} that the algorithm-dependent learnability allows the utilization of the algorithm behavior on the learned function $\tilde{f}$ (the fourth `$\leq$' in Eq.~(\refeq{inter-step})) while almost maintaining the required monotone approximately submodular property (the first three and the last `$\leq$' in Eq.~(\refeq{inter-step})).

\subsubsection{Sufficient for Offline Unconstrained Non-monotone Submodular Maximization with Local Search}\label{subsec-submodular-local-search}

Next, we study the problem of non-monotone submodular function maximization without constraints. As presented in Definition~\ref{def-Prob-nonmonotone-1}, it is to find a subset $X\subseteq V$ maximizing a non-monotone and submodular function $f$. We assume without loss of generality that $f$ is non-negative. This problem has many applications such as the classical NP-hard problem, maximum cut~\cite{goemans1995improved}.

\begin{definition}[Non-monotone Submodular Function Maximization without Constraints]\label{def-Prob-nonmonotone-1}
Given a non-monotone and submodular function $f: 2^V \rightarrow \mathbb{R}^+$ (where $\mathbb{R}^{+}$ denotes the set of non-negative reals), the goal is to find a subset $X^*\subseteq V$ such that
$$
X^*\in\arg \max\nolimits_{X \subseteq V} f(X).
$$
\end{definition}

It has been proved that the local search algorithm can achieve an approximation ratio of $1/3-\epsilon/n$~\cite{feige2011maximizing}. As presented in Algorithm~\ref{alg:local-search}, it starts from the best single element (line~1), and repeatedly adds (lines~2--4) or deletes (lines~5--7) one element until the subset cannot be improved by a factor larger than $1+\epsilon/n^2$. Finally, it outputs the better one between the found subset $X$ and its complement $V\setminus X$ in line~8. 

\begin{algorithm}[htbp]
\caption{Local Search Algorithm~\cite{feige2011maximizing}}
\label{alg:local-search}
\textbf{Input}: all elements $V=\{v_1,\ldots,v_n\}$, and objective function $f$\\
\textbf{Output}: a subset of $V$\\
\textbf{Process}:
\begin{algorithmic}[1]
\STATE Let $v^*\in \arg \max_{v \in V} f(v)$, and $X=\{v^*\}$;
\IF { $\exists v \in V \setminus X$ such that $f(X \cup \{v\}) > (1+\frac{\epsilon}{n^2})\cdot f(X)$ }
\STATE $X=X \cup \{v\}$, and repeat from line~2
\ENDIF
\IF { $\exists v \in X$ such that $f(X \setminus \{v\}) > (1+\frac{\epsilon}{n^2})\cdot f(X)$ }
\STATE $X=X \setminus \{v\}$, and repeat from line~2
\ENDIF
\STATE \textbf{return} the better one between $X$ and $V \setminus X$
\end{algorithmic}
\end{algorithm}

Theorem~\ref{theo-submodular-unconstrained} shows that the problem of non-monotone submodular maximization without constraints is $\alpha$-offline-optimizable, when the algorithm-dependent learnability with respect to the local search algorithm is satisfied. The approximation ratio $\alpha=\frac{1-\beta}{1+\beta}\cdot\left(3+2n\left(\frac{1+\beta}{1-\beta}\left(1+\frac{\epsilon}{n^2}\right)-1\right)\right)^{-1}$. When $\beta=0$, $\alpha$ becomes $1/(3+2\epsilon/n)$, which is slightly better than the approximation ratio of $\frac{1}{3}-\frac{\epsilon}{n}$ achieved by local search under the value oracle model where each solution can be exactly evaluated~\cite{feige2011maximizing}. In the proof of Theorem~\ref{theo-submodular-unconstrained}, we also show that the running time complexity required to achieve the approximation ratio is $O\left(\frac{n^3}{\epsilon} \log\left(\frac{1+\beta}{1-\beta}n\right) \right)$.

\begin{theorem}\label{theo-submodular-unconstrained}
If the problem of non-monotone submodular maximization without constraints is algorithm-dependent learnable with respect to the local search algorithm on distribution $\mathcal{D}$, then it is $\alpha$-offline-optimizable over $\mathcal{D}$, where $$\alpha=\frac{1-\beta}{1+\beta}\cdot\left(3+2n\left(\frac{1+\beta}{1-\beta}\left(1+\frac{\epsilon}{n^2}\right)-1\right)\right)^{-1}.$$ That is, Eq.~(\refeq{eq-offline-opt-multi}) holds with such an $\alpha$.
\end{theorem}
\begin{proof}
By the algorithm-dependent learnability with respect to the local search algorithm, we know from Definition~\ref{def-algo-learnability} that there is a learning algorithm which can return a function $\tilde{f}$ with probability $1-\delta$: 
\begin{equation}
\begin{aligned}\label{proof-eq-algo-learnability-local}
 \forall X \; \text{visited by running local search on} \; \tilde{f}: (1-\beta) \cdot f(X)\leq \tilde{f}(X)\leq  (1+\beta)\cdot f(X),
\end{aligned}
\end{equation}
when getting access to a polynomial number of samples $\{\bm{x}_i,f(\bm{x}_i)\}^m_{i=1}$ drawn i.i.d. from the distribution $\mathcal{D}$. In Eq.~(\refeq{proof-eq-algo-learnability-local}), $\beta \in [0,1)$, and $f$ denotes the true function. Note that a solution visited by local search (i.e., Algorithm~\ref{alg:local-search}) actually means that it is evaluated during the process of running Algorithm~\ref{alg:local-search}.

We consider the local search algorithm running on the learned function $\tilde{f}$ that satisfies Eq.~(\refeq{proof-eq-algo-learnability-local}). Let $X_{\mathrm{local}}$ denote the local optimal solution obtained after finishing the run of lines~1--7 of Algorithm~\ref{alg:local-search}. We first have
\begin{equation}
\begin{aligned}\label{eq-local-1}
\forall S \subseteq X_{\mathrm{local}}: f(S) \leq \left(n\left(\frac{1+\beta}{1-\beta}\left(1+\frac{\epsilon}{n^2}\right)-1\right)+1\right)\cdot f(X_{\mathrm{local}}).
\end{aligned}
\end{equation}
We are then to prove Eq.~(\refeq{eq-local-1}). Let $S=S_1 \subseteq S_2 \subseteq \cdots \subseteq S_k=X_{\mathrm{local}}$, where $S_i \setminus S_{i-1}=\{e_i\}$. We have
\begin{equation}
\begin{aligned}\label{eq-local-2}
&f(S_i)-f(S_{i-1}) \\
&\geq f(X_{\mathrm{local}})-f(X_{\mathrm{local}}\setminus \{e_i\}) & [\text{by the submodularity of $f$}]\\
&\geq f(X_{\mathrm{local}})-\frac{1}{1-\beta}\tilde{f}(X_{\mathrm{local}}\setminus \{e_i\}) &[\text{by Eq.~(\refeq{proof-eq-algo-learnability-local}) as $X_{\mathrm{local}}\setminus \{e_i\}$ is visited by Algorithm~\ref{alg:local-search}}]\\
& \geq f(X_{\mathrm{local}})-\frac{1+\epsilon/n^2}{1-\beta}\tilde{f}(X_{\mathrm{local}}) &[\text{by lines~5--7 of Algorithm~\ref{alg:local-search} and $X_{\mathrm{local}}$'s local optimality}]\\
& \geq \left(1-\frac{1+\beta}{1-\beta}\left(1+\frac{\epsilon}{n^2}\right)\right)\cdot f(X_{\mathrm{local}}) &[\text{by Eq.~(\refeq{proof-eq-algo-learnability-local}) as $X_{\mathrm{local}}$ is visited by Algorithm~\ref{alg:local-search}}]
\end{aligned}
\end{equation}
By summing up $f(S_i)-f(S_{i-1})$ from $i=2$ to $k$, we can get
\begin{equation}
\begin{aligned}\label{eq-local-3}
f(X_{\mathrm{local}})-f(S)=\sum^{k}_{i=2} \left(f(S_i)-f(S_{i-1})\right) &\geq (k-1)\cdot \left(1-\frac{1+\beta}{1-\beta}\left(1+\frac{\epsilon}{n^2}\right)\right)\cdot f(X_{\mathrm{local}})\\
&\geq n\cdot \left(1-\frac{1+\beta}{1-\beta}\left(1+\frac{\epsilon}{n^2}\right)\right)\cdot f(X_{\mathrm{local}}),
\end{aligned}
\end{equation}
leading to Eq.~(\refeq{eq-local-1}).

Similarly, we can derive
\begin{equation}
\begin{aligned}\label{eq-local-4}
\forall X_{\mathrm{local}} \subseteq S: f(S) \leq \left(n\left(\frac{1+\beta}{1-\beta}\left(1+\frac{\epsilon}{n^2}\right)-1\right)+1\right)f(X_{\mathrm{local}}).
\end{aligned}
\end{equation}
Let $X_{\mathrm{local}}=S_1 \subseteq S_2 \subseteq \cdots \subseteq S_k=S$, where $S_i \setminus S_{i-1}=\{e_i\}$. We have
\begin{equation}
\begin{aligned}\label{eq-local-5}
&f(S_i)-f(S_{i-1}) \\
&\leq f(X_{\mathrm{local}} \cup \{e_i\})-f(X_{\mathrm{local}}) & [\text{by the submodularity of $f$}]\\
&\leq \frac{1}{1-\beta}\tilde{f}(X_{\mathrm{local}}\cup \{e_i\})-f(X_{\mathrm{local}}) &[\text{by Eq.~(\refeq{proof-eq-algo-learnability-local}) as $X_{\mathrm{local}}\cup \{e_i\}$ is visited by Algorithm~\ref{alg:local-search}}]\\
& \leq \frac{1+\epsilon/n^2}{1-\beta}\tilde{f}(X_{\mathrm{local}})-f(X_{\mathrm{local}}) &[\text{by lines~2--4 of Algorithm~\ref{alg:local-search} and $X_{\mathrm{local}}$'s local optimality}]\\
& \leq \left(\frac{1+\beta}{1-\beta}\left(1+\frac{\epsilon}{n^2}\right)-1\right)\cdot f(X_{\mathrm{local}}) &[\text{by Eq.~(\refeq{proof-eq-algo-learnability-local}) as $X_{\mathrm{local}}$ is visited by Algorithm~\ref{alg:local-search}}]
\end{aligned}
\end{equation}
By summing up $f(S_i)-f(S_{i-1})$ from $i=2$ to $k$, we can get
\begin{equation}
\begin{aligned}\label{eq-local-6}
f(S)-f(X_{\mathrm{local}})=\sum^{k}_{i=2} \left(f(S_i)-f(S_{i-1})\right) &\leq (k-1)\cdot \left(\frac{1+\beta}{1-\beta}\left(1+\frac{\epsilon}{n^2}\right)-1\right)\cdot f(X_{\mathrm{local}})\\
&\leq n\cdot \left(\frac{1+\beta}{1-\beta}\left(1+\frac{\epsilon}{n^2}\right)-1\right)\cdot f(X_{\mathrm{local}}),
\end{aligned}
\end{equation}
leading to Eq.~(\refeq{eq-local-4}).

Let $X^*$ denote an optimal subset, i.e., $f(X^*)= \max_{X \subseteq V}f(X)$. By utilizing Eqs.~(\refeq{eq-local-1}) and~(\refeq{eq-local-4}), we have 
\begin{equation}
\begin{aligned}\label{eq-local-7}
 &\left(n\left(\frac{1+\beta}{1-\beta}\left(1+\frac{\epsilon}{n^2}\right)-1\right)+1\right)\cdot f(X_{\mathrm{local}}) \geq f(X_{\mathrm{local}} \cap X^*);\\
  &\left(n\left(\frac{1+\beta}{1-\beta}\left(1+\frac{\epsilon}{n^2}\right)-1\right)+1\right)\cdot f(X_{\mathrm{local}}) \geq f(X_{\mathrm{local}} \cup X^*),\\
\end{aligned}
\end{equation}
which leads to
\begin{equation}
\begin{aligned}\label{eq-local-8}
&2\left(n\left(\frac{1+\beta}{1-\beta}\left(1+\frac{\epsilon}{n^2}\right)-1\right)+1\right)\cdot f(X_{\mathrm{local}}) + f(V\setminus X_{\mathrm{local}})\\
&\geq f(X_{\mathrm{local}} \cap X^*)+ f(X_{\mathrm{local}} \cup X^*)+f(V\setminus X_{\mathrm{local}})\\
&\geq f(X_{\mathrm{local}} \cap X^*)+ f(X^* \setminus X_{\mathrm{local}})+f(V)&[\text{by the submodularity of $f$}]\\
&\geq f(X_{\mathrm{local}} \cap X^*)+ f(X^* \setminus X_{\mathrm{local}})&[\text{by the non-negativity of $f$}]\\
&\geq f(X^*)+f(\emptyset)&[\text{by the submodularity of $f$}]\\
&\geq f(X^*)&[\text{by the non-negativity of $f$}]
\end{aligned}
\end{equation}
Then, we can conclude that
\begin{equation}
\begin{aligned}
\max\{f(X_{\mathrm{local}}), f(V\setminus X_{\mathrm{local}})\}  \geq \left(3+2n\left(\frac{1+\beta}{1-\beta}\left(1+\frac{\epsilon}{n^2}\right)-1\right)\right)^{-1}\cdot f(X^*). 
\end{aligned}
\end{equation}
As the local search algorithm in Algorithm~\ref{alg:local-search} returns the better one between $X_{\mathrm{local}}$ and $V\setminus X_{\mathrm{local}}$ in line~8, the returned subset $\tilde{X}$ satisfies
\begin{equation}
\begin{aligned}
f(\tilde{X}) \geq \frac{1}{1+\beta} \cdot \tilde{f}(\tilde{X})&=\frac{1}{1+\beta} \cdot \max\{\tilde{f}(X_{\mathrm{local}}), \tilde{f}(V\setminus X_{\mathrm{local}})\}\\
&\geq \frac{1-\beta}{1+\beta}\cdot \max\{f(X_{\mathrm{local}}), f(V\setminus X_{\mathrm{local}})\}\\
&\geq \frac{1-\beta}{1+\beta}\cdot \left(3+2n\left(\frac{1+\beta}{1-\beta}\left(1+\frac{\epsilon}{n^2}\right)-1\right)\right)^{-1} \cdot f(X^*), 
\end{aligned}
\end{equation}
i.e., Eq.~(\refeq{eq-offline-opt-multi}) in Definition~\ref{def-offline-optimizable} holds with
$$
\alpha=\frac{1-\beta}{1+\beta}\cdot\left(3+2n\left(\frac{1+\beta}{1-\beta}\left(1+\frac{\epsilon}{n^2}\right)-1\right)\right)^{-1}.
$$

Finally, we analyze the time complexity of running the local search algorithm on the learned function $\tilde{f}$. Assume that the local search algorithm has performed $k$ iterations until finding $X_{\mathrm{local}}$, i.e., the maintained subset has been improved $k$ times by local search. Because the local search algorithm starts from the best single element $\tilde{v}^*\in\arg \max_{v \in V} \tilde{f}(v)$ in line~1 of Algorithm~\ref{alg:local-search}, and each improvement improves the $\tilde{f}$ value of the current subset by a factor larger than $1+\epsilon/n^2$, we have
\begin{equation}
\begin{aligned}\label{eq-local-9}
\tilde{f}(X_{\mathrm{local}}) > \left(1+\frac{\epsilon}{n^2}\right)^k\cdot \tilde{f}(\tilde{v}^*).
\end{aligned}
\end{equation}
Let $v^*\in\arg \max_{v \in V} f(v)$ denote the best single element for the true function $f$. We have
\begin{equation}
\begin{aligned}\label{eq-local-10}
\tilde{f}(\tilde{v}^*) \geq \tilde{f}(v^*) \geq (1-\beta) \cdot f(v^*), 
\end{aligned}
\end{equation}
where the last inequality holds by Eq.~(\refeq{proof-eq-algo-learnability-local}) as $v^*$ is visited by Algorithm~\ref{alg:local-search} during the process of determining the best single element in line~1. Furthermore, we have
\begin{equation}
\begin{aligned}\label{eq-local-11}
\tilde{f}(X_{\mathrm{local}}) \leq (1+\beta)\cdot f(X_{\mathrm{local}}) \leq (1+\beta)\cdot nf(v^*), 
\end{aligned}
\end{equation}
where the first inequality holds by Eq.~(\refeq{proof-eq-algo-learnability-local}) as $X_{\mathrm{local}}$ is visited by Algorithm~\ref{alg:local-search}, and the last inequality holds by the submodularity of $f$ and $v^*\in \arg \max_{v \in V} f(v)$. By combining Eqs.~(\refeq{eq-local-9}),~(\refeq{eq-local-10}) and~(\refeq{eq-local-11}), we have 
\begin{equation}
\begin{aligned}
\left(1+\frac{\epsilon}{n^2}\right)^k\cdot (1-\beta) \cdot f(v^*) \leq (1+\beta)\cdot nf(v^*),
\end{aligned}
\end{equation}
leading to $k = O\left(\frac{n^2}{\epsilon} \log\left(\frac{1+\beta}{1-\beta}n\right) \right)$. As each iteration of Algorithm~\ref{alg:local-search} evaluates at most $n$ subsets (by adding a new element to the current subset or deleting an existing element from the current subset), the running time complexity is $O\left(\frac{n^3}{\epsilon} \log\left(\frac{1+\beta}{1-\beta}n\right) \right)$.
\end{proof}

The proof of Theorem~\ref{theo-submodular-unconstrained} is adapted from the proof for the approximation ratio of local search under the value oracle model, i.e., Theorem~3.4 in~\cite{feige2011maximizing}, by using Eq.~(\ref{proof-eq-algo-learnability-local}) under algorithm-dependent learnability to connect the true submodular function $f$ and the learned function $\tilde{f}$. That is, the algorithm-dependent learnability allows the utilization of the local search behavior on the learned function $\tilde{f}$ (e.g., the third `$\geq$' in Eq.~(\refeq{eq-local-2}) and the third `$\leq$' in Eq.~(\refeq{eq-local-5})) while maintaining the required submodular property approximately (e.g., the first two and the last `$\geq$' in Eq.~(\refeq{eq-local-2}) as well as the first two and the last `$\leq$' in Eq.~(\refeq{eq-local-5})).

\subsubsection{Sufficient for Offline Convex Minimization with Projected Gradient Descent}\label{subsec-convex-gradient}

Finally, we study offline convex minimization. Let $\mathcal{X}\subseteq\mathbb{R}^d$ be a non-empty closed convex domain with finite diameter $D=\sup_{\bm{x},\bm{y}\in\mathcal{X}}\|\bm{x}-\bm{y}\|$. We consider a real-valued objective $f$ that is convex and continuously differentiable on an open neighborhood of $\mathcal{X}$, with $\|\nabla f(\bm{x})\|\leq G$ for all $\bm{x}\in\mathcal{X}$ and some constant $G>0$. Requiring differentiability on a neighborhood makes the ambient gradient well-defined even at boundary points of $\mathcal{X}$. Convex minimization in Definition~\ref{def-Prob-convex} is a basic model for many machine learning problems~\cite{nesterov2013introductory}. Meanwhile, as shown in Theorem~\ref{theo-negative-convex}, even when a class of convex functions is PAC learnable and can be minimized efficiently with exact function access, PAC learnability alone still cannot guarantee a non-trivial additive approximation in the offline optimization setting.

\begin{definition}[Convex Minimization]\label{def-Prob-convex}
Given a non-empty closed convex domain $\mathcal{X}\subseteq\mathbb{R}^d$ with finite diameter $D$ and a real-valued function $f$ that is convex and continuously differentiable on an open neighborhood of $\mathcal{X}$, the goal is to find
\[
    \bm{x}^*\in \arg\min_{\bm{x}\in\mathcal{X}} f(\bm{x}).
\]
\end{definition}

It is well known that projected gradient descent can achieve an additive approximation guarantee for convex minimization under the first-order oracle model where the exact gradient can be evaluated~\cite{nesterov2013introductory}. As presented in Algorithm~\ref{alg:gradient-descent}, starting from an initial solution $\bm{x}_1\in\mathcal{X}$, the algorithm repeatedly moves along the negative gradient direction and projects the updated solution back to $\mathcal{X}$, where $\Pi_{\mathcal{X}}$ denotes the Euclidean projection onto $\mathcal{X}$. The algorithm returns the arithmetic mean of its first $T$ iterates. For a convex objective $f$ with $\|\nabla f(\bm{x})\|\leq G$, setting $\eta=D/(G\sqrt{T})$ guarantees that $f(\bar{\bm{x}}_T)-\min_{\bm{x}\in\mathcal{X}}f(\bm{x})\leq DG/\sqrt{T}$.

\begin{algorithm}[htbp]
\caption{Projected Gradient Descent~\cite{nesterov2013introductory}}
\label{alg:gradient-descent}
\textbf{Input}: initial solution $\bm{x}_1 \in \mathcal{X}$, objective function $f$, step size $\eta$, and number of iterations $T\geq1$\\
\textbf{Output}: a solution in $\mathcal{X}$\\
\textbf{Process}:
\begin{algorithmic}[1]
\STATE Let $t=1$ and $\bm{x}_1 \in \mathcal{X}$;
\STATE \textbf{repeat}
\STATE \quad Let $\bm{x}_{t+1}=\Pi_{\mathcal{X}}[\bm{x}_t - \eta \nabla f(\bm{x}_t)]$, and $t=t+1$
\STATE \textbf{until} $t > T$
\STATE \textbf{return} $\bar{\bm{x}}_T = \frac{1}{T}\sum_{t=1}^{T}\bm{x}_t$
\end{algorithmic}
\end{algorithm}

Theorem~\ref{theo-convex-gradient} shows that the first-order algorithm-dependent learnability in Definition~\ref{def-first-order-algo-learnability} is sufficient for offline convex minimization. Unlike the two  algorithms above which require function value evaluations, projected gradient descent determines its trajectory through gradients. Definition~\ref{def-first-order-algo-learnability} therefore controls precisely the first-order information consumed along that trajectory. The resulting additive error contains the usual optimization term and the trajectory-wise gradient error $\zeta_g$. When $\zeta_g=0$, the bound recovers the standard  guarantee $DG/\sqrt{T}$ under exact first-order access.

\begin{theorem}\label{theo-convex-gradient}
If the problem of convex minimization (having the uniform gradient bound $\|\nabla f(\bm{x})\|\leq G$ for all $f\in\mathcal F$ and $\bm{x}\in\mathcal{X}$) is first-order algorithm-dependent learnable with respect to the projected gradient descent algorithm (with step size $\eta=\frac{D}{(G+\zeta_g)\sqrt{T}}$) on distribution $\mathcal{D}$, then it is $\alpha$-offline-optimizable over $\mathcal{D}$, where $$\alpha=\frac{(G+\zeta_g)D}{\sqrt{T}}+\zeta_gD.$$ That is, Eq.~(\refeq{eq-offline-opt-add}) holds with such an $\alpha$.
\end{theorem}
\begin{proof}
By the first-order algorithm-dependent learnability with respect to the projected gradient descent algorithm (with step size $\eta=\frac{D}{(G+\zeta_g)\sqrt{T}}$), we know from Definition~\ref{def-first-order-algo-learnability} that there is a learning algorithm which can return a function $\tilde{f}$ with probability $1-\delta$: 
\begin{equation}
\label{eq:gd-gradient-condition}
    \|\nabla\tilde{f}(\hat{\bm{x}}_t)-\nabla f(\hat{\bm{x}}_t)\|\leq \zeta_g,\qquad t=1,\ldots,T,
\end{equation}
when getting access to a polynomial number of samples $\{\bm{x}_i,f(\bm{x}_i)\}^m_{i=1}$ drawn i.i.d. from the distribution $\mathcal{D}$. Note that $\hat{\bm{x}}_1,\ldots,\hat{\bm{x}}_{T}$ denote the trajectory generated by running the projected gradient descent algorithm (with $\eta=\frac{D}{(G+\zeta_g)\sqrt{T}}$) on $\tilde{f}$. 

Let $\bm{x}^*$ denote an optimal solution, i.e., $f(\bm{x}^*)=\min_{\bm{x}\in\mathcal{X}}f(\bm{x})$. As $\hat{\bm{x}}_t$ denotes the solution after $t-1$ update steps on $\tilde{f}$, we have
\[
    \hat{\bm{x}}_{t+1}=\Pi_{\mathcal{X}}\left[\hat{\bm{x}}_t-\eta\nabla\tilde{f}(\hat{\bm{x}}_t)\right],
    \qquad
    \eta=\frac{D}{(G+\zeta_g)\sqrt{T}} .
\]
Because $\bm{x}^*\in\mathcal{X}$, $\Pi_{\mathcal{X}}(\bm{x}^*)=\bm{x}^*$. The non-expansiveness of Euclidean projection onto a closed convex set therefore gives
\begin{equation}\label{eq-gd-step1}
\begin{aligned}
\|\hat{\bm{x}}_{t+1}-\bm{x}^*\|^2 
&= \left\|\Pi_{\mathcal{X}}\left[\hat{\bm{x}}_t-\eta\nabla\tilde{f}(\hat{\bm{x}}_t)\right]-\Pi_{\mathcal{X}}(\bm{x}^*)\right\|^2\\
&\leq \left\|\hat{\bm{x}}_t-\eta\nabla\tilde{f}(\hat{\bm{x}}_t)-\bm{x}^*\right\|^2\\
&= \|\hat{\bm{x}}_t-\bm{x}^*\|^2-2\eta\langle \nabla\tilde{f}(\hat{\bm{x}}_t),\hat{\bm{x}}_t-\bm{x}^*\rangle+\eta^2\|\nabla\tilde{f}(\hat{\bm{x}}_t)\|^2.
\end{aligned}
\end{equation}
We then bound the inner product term as
\begin{equation}\label{eq-gd-step2}
\begin{aligned}
&\langle \nabla\tilde{f}(\hat{\bm{x}}_t),\hat{\bm{x}}_t-\bm{x}^*\rangle\\
&= \langle \nabla f(\hat{\bm{x}}_t),\hat{\bm{x}}_t-\bm{x}^*\rangle+\langle \nabla\tilde{f}(\hat{\bm{x}}_t)-\nabla f(\hat{\bm{x}}_t),\hat{\bm{x}}_t-\bm{x}^*\rangle\\
&\geq f(\hat{\bm{x}}_t)-f(\bm{x}^*)-\|\nabla\tilde{f}(\hat{\bm{x}}_t)-\nabla f(\hat{\bm{x}}_t)\|\cdot\|\hat{\bm{x}}_t-\bm{x}^*\| & [\text{by the convexity of $f$ and \textit{Cauchy--Schwarz}}]\\
&\geq f(\hat{\bm{x}}_t)-f(\bm{x}^*)-\zeta_gD, & [\text{by Eq.~(\refeq{eq:gd-gradient-condition}) and the definition of $D$}]
\end{aligned}
\end{equation}
and the gradient norm as
\begin{equation}\label{eq-gd-step3}
\begin{aligned}
\|\nabla\tilde{f}(\hat{\bm{x}}_t)\|
&\leq \|\nabla f(\hat{\bm{x}}_t)\|+\|\nabla\tilde{f}(\hat{\bm{x}}_t)-\nabla f(\hat{\bm{x}}_t)\| & [\text{by the triangle inequality}]\\
&\leq G+\zeta_g. & [\text{by the gradient bound of $f$ and Eq.~(\refeq{eq:gd-gradient-condition})}]
\end{aligned}
\end{equation}
Substituting Eqs.~(\refeq{eq-gd-step2}) and~(\refeq{eq-gd-step3}) into Eq.~(\refeq{eq-gd-step1}) and rearranging, we get
\begin{equation}\label{eq-gd-step4}
2\eta(f(\hat{\bm{x}}_t)-f(\bm{x}^*))\leq \|\hat{\bm{x}}_t-\bm{x}^*\|^2-\|\hat{\bm{x}}_{t+1}-\bm{x}^*\|^2+2\eta\zeta_gD+\eta^2(G+\zeta_g)^2.
\end{equation}
Summing Eq.~(\refeq{eq-gd-step4}) from $t=1$ to $T$ gives
\begin{equation}\label{eq-gd-step5}
\begin{aligned}
2\eta\sum^{T}_{t=1}(f(\hat{\bm{x}}_t)-f(\bm{x}^*))
&\leq \|\hat{\bm{x}}_1-\bm{x}^*\|^2-\|\hat{\bm{x}}_{T+1}-\bm{x}^*\|^2+2T\eta\zeta_gD+T\eta^2(G+\zeta_g)^2\\
&\leq D^2+2T\eta\zeta_gD+T\eta^2(G+\zeta_g)^2,
\end{aligned}
\end{equation}
where the last inequality holds because $\hat{\bm{x}}_1,\bm{x}^*\in\mathcal{X}$ and $\|\hat{\bm{x}}_{T+1}-\bm{x}^*\|^2\geq 0$. Let $\bar{\bm{x}}_T=\sum_{t=1}^T\hat{\bm{x}}_t/T$ be the solution returned by Algorithm~\ref{alg:gradient-descent}. Since $\mathcal{X}$ is convex, $\bar{\bm{x}}_T\in\mathcal{X}$. By the convexity of $f$, Jensen's inequality, and Eq.~(\refeq{eq-gd-step5}), we have
\begin{equation}\label{eq-gd-step6}
\begin{aligned}
f(\bar{\bm{x}}_T)-f(\bm{x}^*)
&\leq \frac{1}{T}\sum^{T}_{t=1}\left(f(\hat{\bm{x}}_t)-f(\bm{x}^*)\right)\leq \frac{D^2}{2\eta T}+\zeta_gD+\frac{\eta(G+\zeta_g)^2}{2}.
\end{aligned}
\end{equation}
With $\eta=D/((G+\zeta_g)\sqrt{T})$, Eq.~(\refeq{eq-gd-step6}) implies
\begin{equation}
f(\bar{\bm{x}}_T)-f(\bm{x}^*) \leq \frac{(G+\zeta_g)D}{\sqrt{T}}+\zeta_gD.
\end{equation}
Thus, Eq.~(\refeq{eq-offline-opt-add}) holds with $\alpha=\frac{(G+\zeta_g)D}{\sqrt{T}}+\zeta_gD$ .
\end{proof}

The proof of Theorem~\ref{theo-convex-gradient} follows the standard analysis of projected gradient descent, using Eq.~(\refeq{eq:gd-gradient-condition}) to connect the learned and true gradients only along the generated trajectory. We can find from the one-step upper-bound analysis in Eqs.~(\refeq{eq-gd-step1})--(\refeq{eq-gd-step4}) that the first-order algorithm-dependent learnability allows the utilization of the algorithm behavior on the learned function $\tilde{f}$ (the first `$=$' in Eq.~(\refeq{eq-gd-step1})), while retaining the first-order convexity relation of the true function $f$ up to an additive error $\zeta_gD$ (the two `$\geq$' in Eq.~(\refeq{eq-gd-step2})) and the gradient bound of $f$ up to an additive error $\zeta_g$ (the second `$\leq$' in Eq.~(\refeq{eq-gd-step3})). Finally, the convexity of $f$ converts the average trajectory performance into a guarantee for the returned mean iterate. Together with the two discrete cases in Theorems~\ref{theo-submodular-constrained} and~\ref{theo-submodular-unconstrained}, this result supports an information-matched view of algorithm-dependent learnability: function values along the generated optimization trajectory must be reliable for value-query optimizers, whereas gradients must be reliable for first-order optimizers.

\section{Offline Optimization Inspired by Algorithm-dependent Learnability}
\label{sec:method}

The sufficiency of algorithm-dependent learnability for the three offline optimization scenarios proved in Section~\ref{sec-learnability} suggests a practical design principle: an offline optimizer should preserve the information needed along the search process it will execute, rather than approximate the objective uniformly over the entire search space. First, the required precision is query-local: it is needed only on the points evaluated by the optimizer rather than uniformly over the entire search space. Besides, the method should be information-matched: the learned object should preserve the signal that determines the optimizer’s next decision, such as function values for value-query methods or gradients for first-order methods. 

In practice, offline optimization is given a fixed offline dataset, while the optimizer-induced query trajectory is unavailable. However, we can construct synthetic trajectories from the fixed offline dataset as an operational proxy for modeling, which motivates the trajectory-learning framework for offline optimization, comprising three stages: trajectory construction, trajectory modeling, and candidate generation. We use the framework to analyze existing trajectory-learning-based methods and find that their trajectory-construction rules do not follow optimizers' decision signal well. Thus, we introduce the concepts of gradient consistency and uncertainty calibration to integrate the optimizers' decision signal, and propose a new trajectory-learning method by instantiating the framework with Uncertainty-aware Gradient-guided Trajectory Learning (UGTL), which can better learn the local search behavior and thus generate better candidate solutions. 

\subsection{Trajectory-Learning Framework}

The analyses in Sections~\ref{subsec-submodular-greedy}--\ref{subsec-convex-gradient} share one structural feature: the approximation guarantee for offline data-driven optimization is tied to the optimizer's trajectory. Motivated by this observation, we introduce a trajectory-learning framework that represents an optimizer's search behavior through candidate improvement paths constructed from offline data. The framework comprises three stages: \emph{trajectory construction}, \emph{trajectory modeling}, and \emph{candidate generation}. Specifically, given an offline dataset $D$, the trajectory-construction stage produces a trajectory dataset $D_{\mathrm{traj}}=\{\tau_j\}_{j=1}^{N_{\mathrm{traj}}}$, where each length-$H$ trajectory takes the form
$\tau_j=\big((\bm{x}_{j,1},y_{j,1}),\ldots,(\bm{x}_{j,H},y_{j,H})\big)$.
The trajectory-modeling stage learns the transition structure encoded by these paths, while the candidate-generation stage uses the learned behavior to produce candidate designs for evaluation. Algorithm~\ref{alg:framework} formalizes this three-stage framework.

\begin{algorithm}[ht]
\caption{Trajectory-Learning Framework for Offline Optimization}
\label{alg:framework}
\textbf{Input}: Offline dataset $D$, candidate budget $Q$\\
\textbf{Output}: Candidate set $C$ with $|C|\leq Q$\\
\textbf{Process}:
\begin{algorithmic}[1]
\STATE \textbf{Trajectory construction.} Construct $D_{\mathrm{traj}}$ from $D$, with each trajectory representing a plausible improvement path.
\STATE \textbf{Trajectory modeling.} Train a model $M_{\theta}$ on $D_{\mathrm{traj}}$ to capture the search behavior encoded by these paths.
\STATE \textbf{Candidate generation.} Use $M_{\theta}$, optionally with a proxy model, to generate and select a set $C$ of candidate solutions.
\STATE \textbf{return} $C$
\end{algorithmic}
\end{algorithm}

Together, the three stages specify how trajectory information is extracted from the offline dataset, represented by a learned model, and converted into candidate designs. Trajectory construction determines which empirical search behavior is visible to the learner; trajectory modeling captures that behavior; and candidate generation executes it, typically under a high-value condition. Since later stages can only exploit the supervision supplied by the constructed paths, the stage of trajectory construction is crucial and incoherent trajectories may cause even a capable model to learn unsuitable search dynamics.

We want to emphasize that the connection to the sufficient condition, algorithm-dependent learnability, for offline optimization proved in Section~\ref{sec-learnability} is motivational rather than a formal implication. Algorithm-dependent learnability is stated using unknown objective information along optimizer-induced queries, whereas a practical method must construct its training trajectories from the offline dataset $D$ alone. The framework therefore operationalizes the principle of concentrating learning on plausible search processes, without verifying the theoretical conditions rigorously.

\subsection{Existing Methods under the Framework}\label{subsec:existing-methods}

In this section, we show that some existing representative offline optimization methods can be contained under our trajectory-learning framework and analyzed through the same three stages. Table~\ref{tab:comparison} maps each method's construction rule, learned trajectory representation, and candidate-generation procedure to Algorithm~\ref{alg:framework}. 

\begin{table}[htbp]
\centering
\caption{Existing trajectory-related offline optimization methods viewed through the trajectory-learning lens.}
\label{tab:comparison}
\small
\setlength{\tabcolsep}{4pt}
\begin{tabularx}{\textwidth}{@{}l>{\raggedright\arraybackslash}X>{\raggedright\arraybackslash}X>{\raggedright\arraybackslash}X@{}}
\toprule
\textbf{Method} & \textbf{Trajectory construction} & \textbf{Trajectory modeling} & \textbf{Candidate generation} \\
\midrule
\multicolumn{4}{l}{\emph{Generative trajectory learning methods}} \\
BONET~\cite{mashkaria2023generative} & Monotone sort-sampled paths with regret budgets & Autoregressive Transformer & Rollout under low-regret conditioning \\
PGS~\cite{chemingui2024offline} & Randomized paths from a top-percentile subset & Offline RL policy with a surrogate & Policy-guided gradient search \\
GTG~\cite{yun2024guided} & Random local score-feasible paths & Conditional diffusion over trajectories & Guided diffusion sampling with proxy filtering \\
\midrule
\multicolumn{4}{l}{\emph{Trajectory-informed optimization methods}} \\
MATCH-OPT~\cite{hoang2024learning} & Monotone binned transitions & Surrogate with trajectory gradient matching & Gradient ascent on the surrogate \\
ROOT~\cite{dao2025root} & Low/high endpoint pairs from GP-posterior mean functions & Probabilistic bridges between paired endpoints & Reverse bridges from top offline designs \\
\bottomrule
\end{tabularx}
\end{table}

This formulation provides a taxonomy for these existing methods that some methods directly model complete trajectories, whereas others use trajectory-derived signals as auxiliary supervision or transport paths.
Specifically, BONET~\cite{mashkaria2023generative} constructs monotone paths and models them autoregressively; PGS~\cite{chemingui2024offline} converts randomly sampled transitions from top-$p$ percentile offline data into policy supervision whose actions are defined relative to a surrogate gradient; and GTG~\cite{yun2024guided} fits a diffusion model to locally score-feasible paths. These methods instantiate all three stages explicitly. MATCH-OPT~\cite{hoang2024learning} uses monotone paths indirectly as gradient-matching supervision, whereas ROOT~\cite{dao2025root} directly learns stochastic transport bridges between low- and high-value endpoint distributions.

The comparison clearly clarifies what the trajectory-construction stage must provide. Increasing scores alone, as constructed by BONET~\cite{mashkaria2023generative} and PGS~\cite{chemingui2024offline}, do not make a sequence an informative improvement path, since sorting or randomly pairing high-value designs might create long transitions that do not describe how an optimizer moves locally~\citep{yun2024guided}. Restricting transitions to neighborhoods like GTG~\cite{yun2024guided} can improve local coherence, but locality alone does not ensure that successive steps follow a consistent improvement direction. It is expected that useful trajectories should encode both local coherence and directional progress.

In Figure~\ref{fig:branin_experiment}, we illustrate this distinction on the negated Branin function over $[-5,10]\times[0,15]$. We construct the offline dataset by uniformly sampling 50{,}000 points and retaining only the bottom $60\%$ according to their objective values. As shown in Figure~\ref{fig:branin_experiment}(a), BONET's sort-sampled paths and PGS's randomly assembled paths contain long, spatially incoherent transitions. Figure~\ref{fig:branin_experiment}(b) shows that GTG substantially improves locality by restricting transitions to score-feasible neighborhoods. However, because its next points are sampled without explicit directional guidance, the resulting trajectories (shown in pink color) do not consistently form smooth improvement paths toward high-value regions. Figure~\ref{fig:branin_experiment}(c) further shows that the geometry of the constructed paths (generated by the same conditional diffusion model and full-trajectory sampler) persists after trajectory modeling, where the GTG trajectories show limited improvement.

\begin{figure}[ht]
    \centering
    \begin{tikzpicture}[baseline=(current bounding box.center)]
        \draw[braninbonet, line width=1.0pt] (0,0)--(0.42,0);
        \filldraw[fill=braninbonet, draw=braninbonet] (0.21,0) circle (1.2pt);
        \node[anchor=west] at (0.50,0) {\scriptsize BONET};
        \draw[braninpgs, line width=1.0pt] (1.90,0)--(2.32,0);
        \filldraw[fill=braninpgs, draw=braninpgs] (2.11,0) circle (1.2pt);
        \node[anchor=west] at (2.45,0) {\scriptsize PGS};
        \draw[braningtg, line width=1.0pt] (3.60,0)--(4.02,0);
        \filldraw[fill=braningtg, draw=braningtg] (3.81,0.08)--(3.72,-0.07)--(3.90,-0.07)--cycle;
        \node[anchor=west] at (4.15,0) {\scriptsize GTG};
        \draw[braninugtl, line width=1.0pt] (5.35,0)--(5.77,0);
        \filldraw[fill=braninugtl, draw=braninugtl] (5.56,0) circle (1.2pt);
        \node[anchor=west] at (5.90,0) {\scriptsize UGTL};
        \fill[black] (7.30,0) circle (1.35pt);
        \node[anchor=west] at (7.47,0) {\scriptsize Start};
        \node[text=red] at (8.90,0) {\scriptsize$\bigstar$};
        \node[anchor=west] at (9.08,0) {\scriptsize Optima};
    \end{tikzpicture}

    \vspace{0.25em}
    \begin{subfigure}[t]{0.318\textwidth}
        \centering
        \includegraphics[width=\linewidth,trim=0 0 145bp 0,clip]{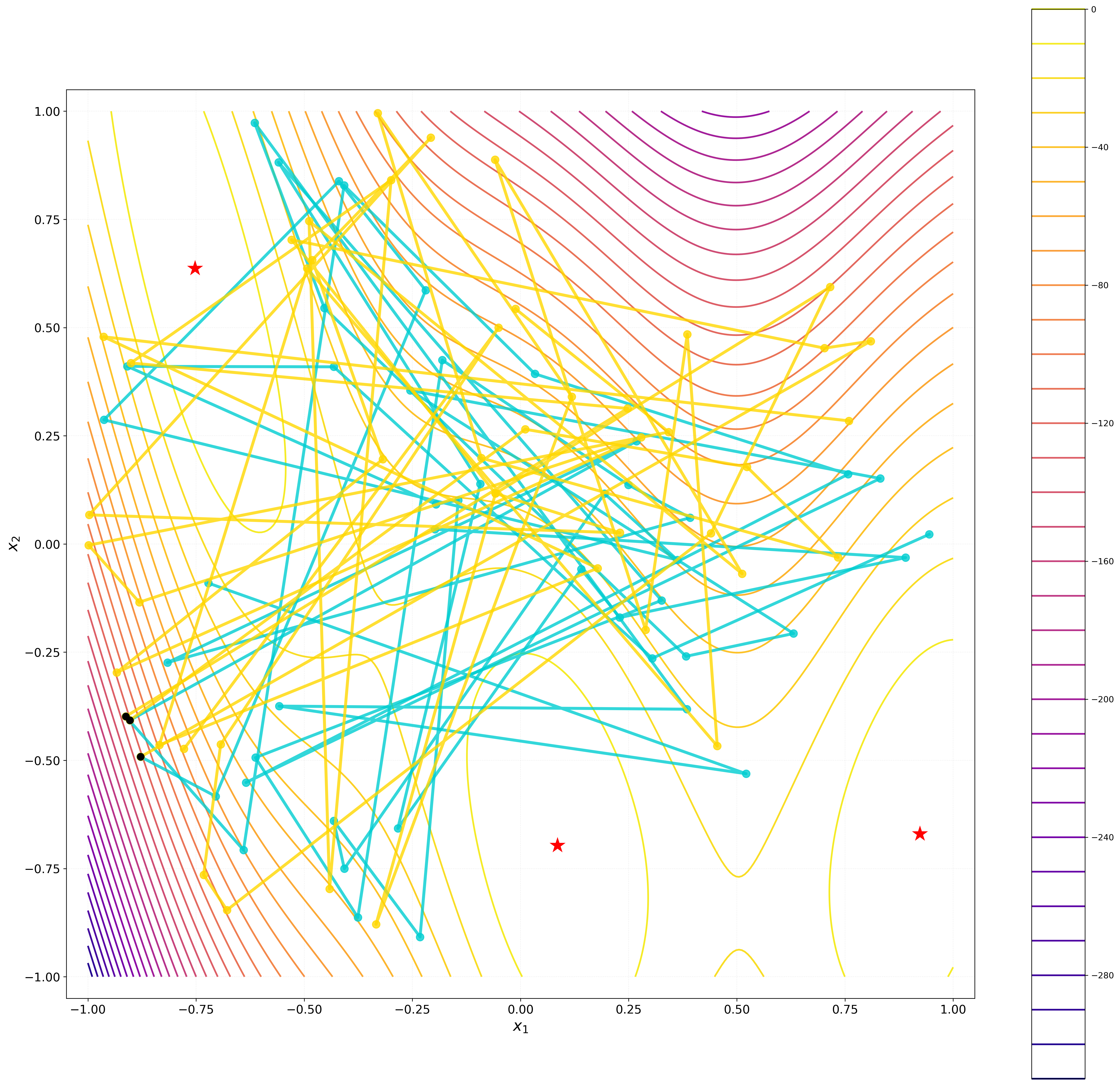}
        \caption{BONET and PGS constructors}
        \label{fig:branin_baselines}
    \end{subfigure}
    \hfill
    \begin{subfigure}[t]{0.318\textwidth}
        \centering
        \includegraphics[width=\linewidth,trim=0 0 145bp 0,clip]{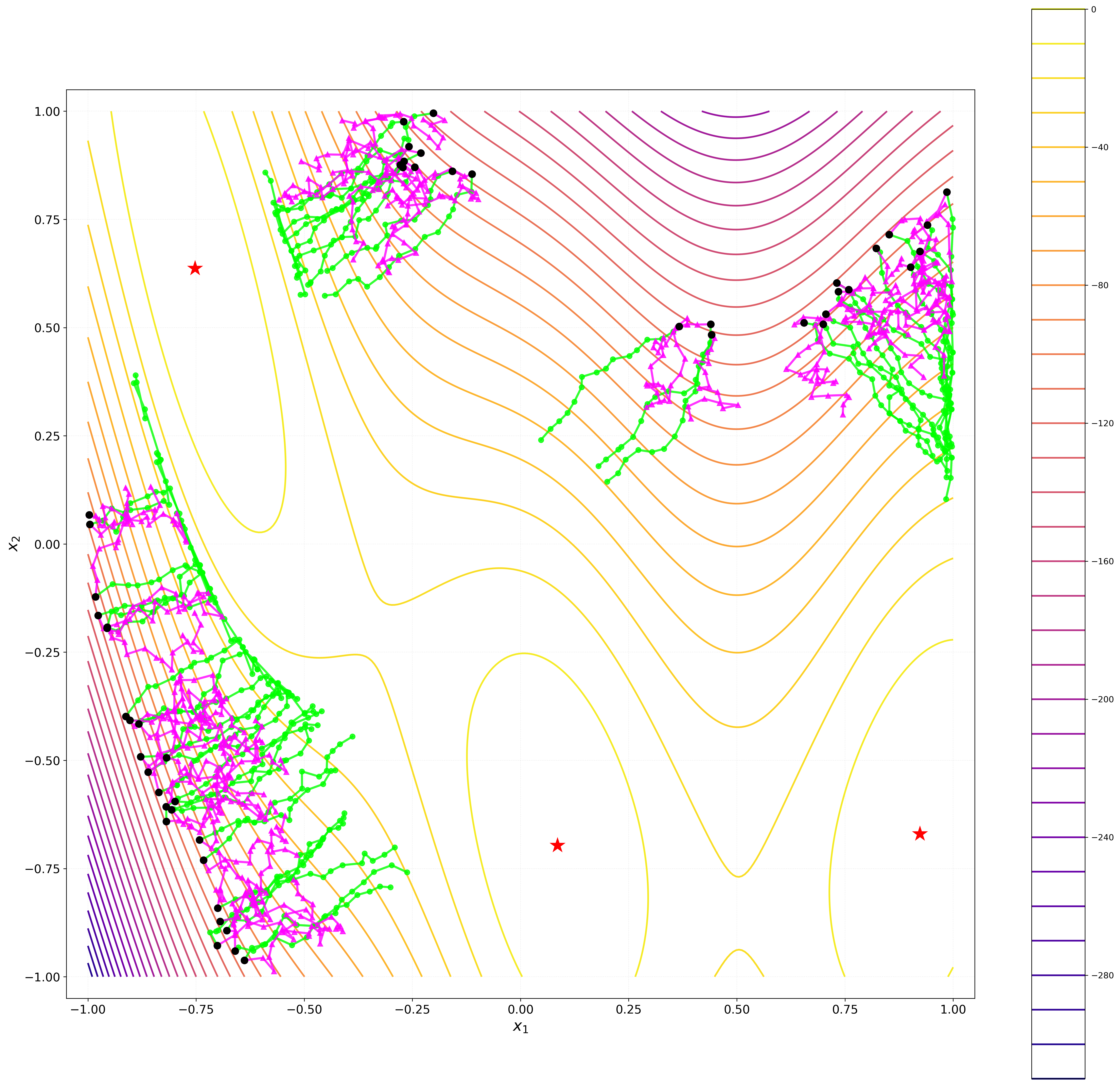}
        \caption{GTG and UGTL constructors}
        \label{fig:branin_construction}
    \end{subfigure}
    \hfill
    \begin{subfigure}[t]{0.318\textwidth}
        \centering
        \includegraphics[width=\linewidth,trim=0 0 145bp 0,clip]{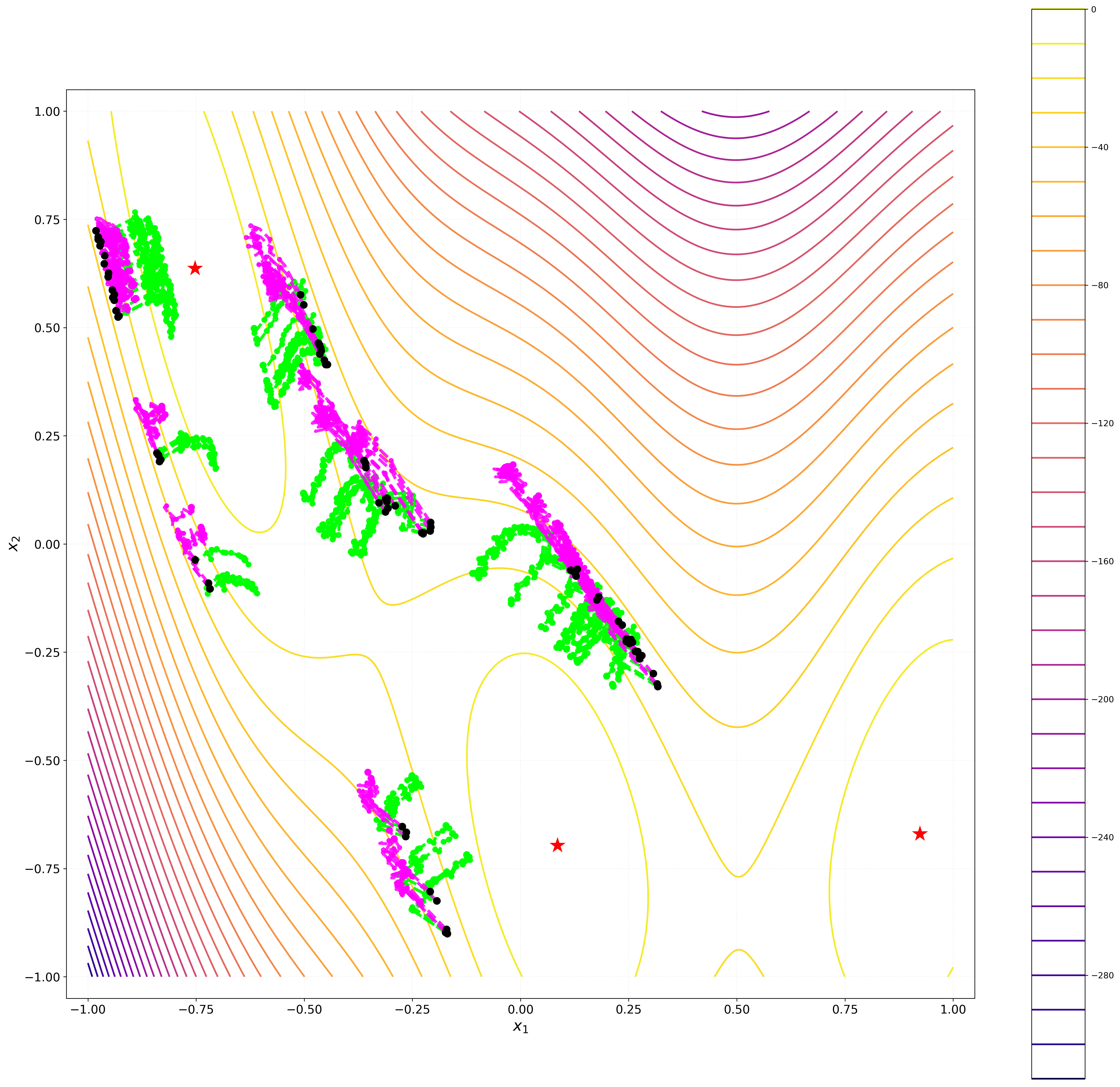}
        \caption{Generated trajectories}
        \label{fig:branin_generated}
    \end{subfigure}
    \caption{A motivating example on the negated Branin function. (a) and (b) compare constructed trajectories, and (c) shows trajectories generated by GTG and UGTL with the same conditional diffusion model. Black dots and red stars denote starting points and global optima, respectively.}
    \label{fig:branin_experiment}
\end{figure}
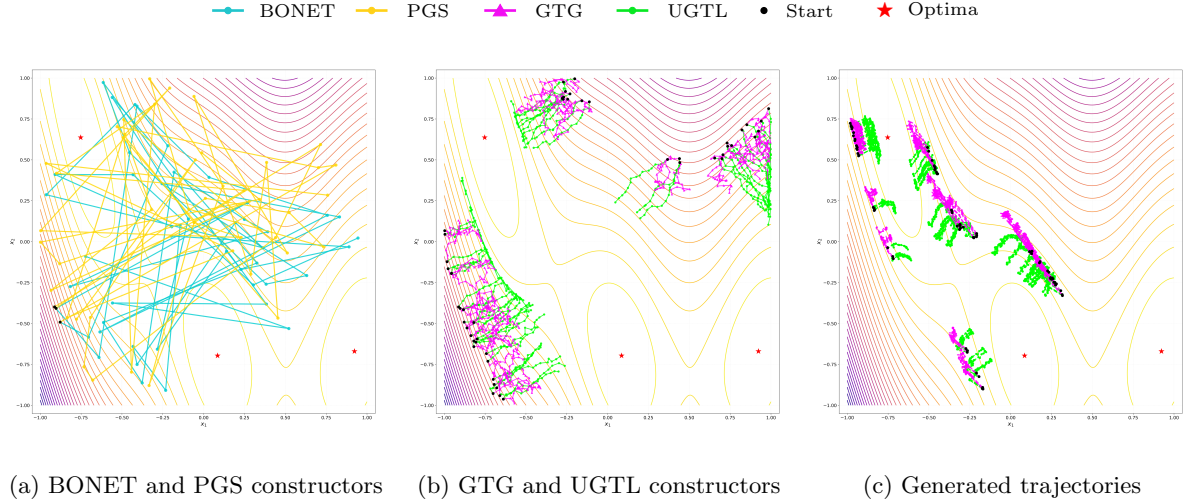

These observations motivate a trajectory constructor that preserves locality while guiding successive transitions toward consistent improvement. We therefore introduce approximate proxy gradients to provide directional guidance and improve trajectory quality, and formulate this design as Uncertainty-aware Gradient-guided Trajectory Learning (UGTL). As illustrated by the green trajectories in Figure~\ref{fig:branin_experiment}(b) and (c), UGTL produces smoother ascent-oriented paths whose structure is retained by the downstream diffusion model. In the next subsection, we will introduce UGTL in detail.

\subsection{A New Offline Optimization Method: Uncertainty-aware Gradient-guided Trajectory Learning}\label{subsec:our-algorithm}

To provide directional guidance of the trajectory construction procedure, we introduce UGTL, which implements the three stages of Algorithm~\ref{alg:framework}. In trajectory construction, it constructs locally coherent paths using proxy-gradient guidance adjusted by the prediction uncertainty of an ensemble of surrogate models. In trajectory modeling, it learns the retained paths using a terminal-score-conditioned diffusion model. In candidate generation, it produces the evaluation batch through guided trajectory sampling and a Cluster-then-Select procedure.

\paragraph{Trajectory construction} Each trajectory
$\tau=((\bm{x}^{\tau}_1,y^{\tau}_1),\ldots,(\bm{x}^{\tau}_H,y^{\tau}_H))$
starts from a point sampled uniformly from the bottom $p_{\mathrm{init}}\%$ of observed scores, leaving scope for improvement. Given the current prefix $\tau_{1:k}$, we define its best observed score as
\begin{equation}
    b_k
    =
    \max_{1\leq h\leq k}
    y^{\tau}_h.
\end{equation}
We then construct a score-feasible local neighborhood
\begin{equation}
    S_k
    =
    \operatorname{NN}_K\!\left(
    \bm{x}^{\tau}_k;
    \left\{
    (\bm{x}_j,y_j)\in D:
    \bm{x}_j
    \notin
    \{\bm{x}^{\tau}_1,\ldots,\bm{x}^{\tau}_k\},
    \;
    y_j\geq b_k-\xi
    \right\}
    \right),
\label{eq:neighbor-set}
\end{equation}
where $\operatorname{NN}_K(\bm{x};\mathcal{S})$ returns up to $K$ nearest designs to $\bm{x}$ in the set $\mathcal{S}$. The score constraint $y_j\geq b_k-\xi$ keeps candidate scores close to the best value reached by the current trajectory, while $\xi\geq0$ allows short non-monotone transitions. If no unvisited design satisfies the score constraint, we use the $K$ highest-scoring unvisited designs as a fallback, breaking score ties by distance to $\bm{x}^{\tau}_k$.

Within $S_k$, we favor transitions aligned with the gradient. For this purpose, we train an ensemble of $M_{\mathrm{ens}}$ surrogate models $\{\tilde f_{\phi_m}\}_{m=1}^{M_{\mathrm{ens}}}$ on $D$. The ensemble mean $\mu_{\phi}$ serves as the proxy score, while the ensemble dispersion $\sigma_{\phi}$ provides an estimate of predictive uncertainty:
\begin{equation}
    \mu_{\phi}(\bm{x})
    =
    \frac{1}{M_{\mathrm{ens}}}
    \sum_{m=1}^{M_{\mathrm{ens}}}
    \tilde f_{\phi_m}(\bm{x}),
    \qquad
    \sigma_{\phi}(\bm{x})
    =
    \sqrt{
    \frac{1}{M_{\mathrm{ens}}}
    \sum_{m=1}^{M_{\mathrm{ens}}}
    \big(
        \tilde f_{\phi_m}(\bm{x})
        -
        \mu_{\phi}(\bm{x})
    \big)^2
    }.
\label{eq:muandsigma}
\end{equation}
We use the gradient of the ensemble mean, $\nabla_{\bm{x}}\mu_{\phi}(\bm{x})$, as a proxy gradient to provide directional guidance during trajectory construction.

For each $(\bm{x}_j,y_j)\in S_k$, we define the stabilized cosine alignment
\begin{equation}
    a_{k,j}
    =
    \frac{
    \left\langle
    \nabla_{\bm{x}}\mu_{\phi}(\bm{x}^{\tau}_k),
    \bm{x}_j-\bm{x}^{\tau}_k
    \right\rangle
    }{
    \|
    \nabla_{\bm{x}}\mu_{\phi}(\bm{x}^{\tau}_k)
    \|
    \|
    \bm{x}_j-\bm{x}^{\tau}_k
    \|
    +
    c_{\mathrm{cos}}
    },
\end{equation}
where $c_{\mathrm{cos}}=10^{-8}$ provides numerical stability. The next point is sampled according to
\begin{equation}
    \Pr\!\left(
    (\bm{x}^{\tau}_{k+1},y^{\tau}_{k+1})
    =
    (\bm{x}_j,y_j)
    \mid
    \tau_{1:k}
    \right)
    =
    \frac{
    \exp(a_{k,j}/\kappa_k)
    }{
    \sum_{(\bm{x}_{j'},y_{j'})\in S_k}
    \exp(a_{k,j'}/\kappa_k)
    }.
\label{eq:softmax}
\end{equation}
The temperature $\kappa_k$ is determined by the ensemble uncertainty:
\begin{equation}
    \kappa_k
    =
    1+
    \frac{
    \sigma_{\phi}(\bm{x}^{\tau}_k)
    }{
    \bar{\sigma}_{\phi}
    +
    c_{\mathrm{temp}}
    },
    \qquad
    \bar{\sigma}_{\phi}
    =
    \frac{1}{N}
    \sum_{i=1}^{N}
    \sigma_{\phi}(\bm{x}_i),
\end{equation}
where $c_{\mathrm{temp}}>0$. Low predictive uncertainty yields a sharper preference for gradient-aligned transitions, whereas high uncertainty increases the temperature and spreads probability across the feasible neighborhood.

After constructing $N_{\mathrm{traj}}$ trajectories, we rank them by their endpoint improvement,
\begin{equation}
    \Delta_y(\tau)
    =
    y^{\tau}_H-y^{\tau}_1,
\end{equation}
and retain the top $\lceil\rho N_{\mathrm{traj}}\rceil$ trajectories, where $\rho\in(0,1]$. This filtering removes stagnant paths without requiring every transition to be monotone. The retained trajectories form $D_{\mathrm{traj}}$, and Algorithm~\ref{algo:traj_construction} summarizes the complete construction stage.

\begin{algorithm}[t]
\caption{UGTL: Uncertainty-aware Gradient-guided Trajectory Construction}
\label{algo:traj_construction}
\textbf{Input}: Offline dataset $D$, horizon $H$, trajectory count $N_{\mathrm{traj}}$, initial percentile $p_{\mathrm{init}}$, neighborhood size $K$, score tolerance $\xi$, retention ratio $\rho$, ensemble size $M_{\mathrm{ens}}$\\
\textbf{Output}: Trajectory dataset $D_{\mathrm{traj}}$
\begin{algorithmic}[1]
\STATE Train the surrogate ensemble $\{\tilde f_{\phi_m}\}_{m=1}^{M_{\mathrm{ens}}}$ on $D$ and compute $\mu_{\phi}(\cdot)$ and $\sigma_{\phi}(\cdot)$ using Eq.~(\refeq{eq:muandsigma});
\STATE Initialize $\mathcal{T}_{\mathrm{traj}}\leftarrow\emptyset$;
\FOR{$u=1,\ldots,N_{\mathrm{traj}}$}
    \STATE Sample $(\bm{x}^{\tau}_1,y^{\tau}_1)$ from the bottom $p_{\mathrm{init}}\%$ of $D$ and initialize $\tau\leftarrow((\bm{x}^{\tau}_1,y^{\tau}_1))$;
    \FOR{$k=1,\ldots,H-1$}
        \STATE Construct $S_k$ using Eq.~(\refeq{eq:neighbor-set});
        \IF{$S_k=\emptyset$}
            \STATE Use the $K$ highest-scoring unvisited designs
        \ENDIF
        \STATE Sample $(\bm{x}^{\tau}_{k+1},y^{\tau}_{k+1})$ from $S_k$ using Eq.~(\refeq{eq:softmax});
        \STATE Append $(\bm{x}^{\tau}_{k+1},y^{\tau}_{k+1})$ to $\tau$
    \ENDFOR
    \STATE Add $\tau$ to $\mathcal{T}_{\mathrm{traj}}$
\ENDFOR
\STATE Rank the trajectories in $\mathcal{T}_{\mathrm{traj}}$ by $\Delta_y(\tau)=y^{\tau}_H-y^{\tau}_1$;
\STATE Retain the top $\lceil\rho N_{\mathrm{traj}}\rceil$ trajectories as $D_{\mathrm{traj}}$
\RETURN $D_{\mathrm{traj}}$
\end{algorithmic}
\end{algorithm}

\paragraph{Trajectory modeling}

For each $\tau\in D_{\mathrm{traj}}$, let
$\bm{X}_{\tau}=[\bm{x}^{\tau}_1,\ldots,\bm{x}^{\tau}_H]$
denote its design sequence and let
$y_{\tau}^{\mathrm{end}}=y^{\tau}_H$
denote its terminal score. We model the design sequence conditioned on its terminal score using a conditional diffusion model
$p_{\theta}(\bm{X}_{\tau}\mid y_{\tau}^{\mathrm{end}})$.
The forward process is
\begin{equation}
    q(
    \bm{X}^{(t)}_{\tau}
    \mid
    \bm{X}^{(0)}_{\tau}
    )
    =
    \mathcal{N}\!\left(
    \sqrt{\bar{\alpha}_t}\,
    \bm{X}^{(0)}_{\tau},
    (1-\bar{\alpha}_t)I
    \right),
\end{equation}
where $t=1,\ldots,T_{\mathrm{diff}}$, $\alpha_t\in(0,1)$ is the signal-retention coefficient at diffusion step $t$, and
$\bar{\alpha}_t=\prod_{s=1}^{t}\alpha_s$.

We train a noise-prediction network
$\varepsilon_{\theta}(\bm{X}^{(t)}_{\tau},t,y)$
using classifier-free conditioning:
\begin{equation}
    \mathcal{L}(\theta)
    =
    \mathbb{E}_{\tau,t,\varepsilon}
    \left[
    \left\|
    \varepsilon
    -
    \varepsilon_{\theta}
    (
    \bm{X}^{(t)}_{\tau},
    t,
 y
    )
    \right\|^2
    \right],
\end{equation}
where
$\tau$ is sampled from $D_{\mathrm{traj}}$,
$t$ is sampled uniformly from
$\{1,\ldots,T_{\mathrm{diff}}\}$,
and $\varepsilon\sim\mathcal{N}(0,I)$.
The condition $y$ equals $y_\tau^{\mathrm{end}}$
with probability $1-p_{\mathrm{drop}}$; otherwise, it
is dropped and represented by a dedicated null-conditioning
embedding, which enables classifier-free guidance during candidate generation.

\paragraph{Candidate generation}

We use the learned diffusion model to generate trajectories toward a target terminal score. Let $y_{\mathrm{tar}}
    =
    \lambda_{\mathrm{tar}} \cdot 
    \max_{(\bm{x},y)\in D}y$, where $\lambda_{\mathrm{tar}}\geq1$ controls the target level. Classifier-free guidance with scale $s_{\mathrm{cfg}}$ computes
\begin{equation}
    \widehat{\varepsilon}_{\theta}
    (
    \bm{X}^{(t)},
    t,
    y_{\mathrm{tar}}
    )
    =
    \varepsilon_{\theta}
    (
    \bm{X}^{(t)},
    t,
    \varnothing
    )
    +
    s_{\mathrm{cfg}}
    \left(
    \varepsilon_{\theta}
    (
    \bm{X}^{(t)},
    t,
    y_{\mathrm{tar}}
    )
    -
    \varepsilon_{\theta}
    (
    \bm{X}^{(t)},
    t,
    \varnothing
    )
    \right).
\label{eq:computeeps}
\end{equation}
Each reverse step is
\begin{equation}
    \bm{X}^{(t-1)}
    =
    \frac{1}{\sqrt{\alpha_t}}
    \left(
    \bm{X}^{(t)}
    -
    \frac{1-\alpha_t}
    {\sqrt{1-\bar{\alpha}_t}}
    \widehat{\varepsilon}_{\theta}
    (
    \bm{X}^{(t)},
    t,
    y_{\mathrm{tar}}
    )
    \right)
    +
    \sigma_t\bm{z}_t,
    \qquad
    \bm{z}_t\sim\mathcal{N}(0,I),
\label{eq:ddpm_reverse}
\end{equation}
where $\sigma_t^2
    =
    \frac{
    1-\bar{\alpha}_{t-1}
    }{
    1-\bar{\alpha}_t
    }
    (1-\alpha_t)$, and $\bm{z}_1=\bm{0}$ at the final reverse step.

For each generated trajectory, we sample
$\tau^{\mathrm{ctx}}\in D_{\mathrm{traj}}$
and use its first $H_{\mathrm{ctx}}<H$ designs as the context
\begin{equation}
    \bm{X}_{\mathrm{ctx}}
    =
    [
    \bm{x}^{\tau^{\mathrm{ctx}}}_1,
    \ldots,
    \bm{x}^{\tau^{\mathrm{ctx}}}_{H_{\mathrm{ctx}}}
    ].
\end{equation}
We keep this prefix fixed by overwriting the corresponding positions after each reverse step:
\begin{equation}
    \bm{X}^{(t-1)}_{1:H_{\mathrm{ctx}}}
    \leftarrow
    \bm{X}_{\mathrm{ctx}}.
\label{eq:tauupdating}
\end{equation}
This anchors each generated trajectory to an observed prefix while allowing its suffix to move toward the target score.

The generated suffixes form a candidate pool $\mathcal{P}$. To improve robustness, instead of directly returning the top-$Q$ proxy-scored designs, we propose Cluster-then-Select, which retains the top $M_{\mathrm{pool}}$ valid candidates under $\mu_{\phi}$, where $M_{\mathrm{pool}}\geq Q$, partitions them into $Q$ clusters, and selects the design with the highest proxy score from each cluster $\mathcal{B}_q$:
\begin{equation}
    \bm{x}^{\mathrm{out}}_q
    =
    \operatorname*{arg\,max}_{\bm{x}\in\mathcal{B}_q}
    \mu_{\phi}(\bm{x}),
    \qquad
    q=1,\ldots,Q.
\end{equation}
Algorithm~\ref{algo:sampling} summarizes the complete candidate-generation stage, where the resulting candidate set combines high proxy scores with coverage across distinct generated regions.

\begin{algorithm}[ht]
\caption{UGTL: Candidate Generation}
\label{algo:sampling}
\textbf{Input}: Trained diffusion model $\varepsilon_{\theta}$, trajectory dataset $D_{\mathrm{traj}}$, proxy $\mu_{\phi}$, context length $H_{\mathrm{ctx}}$, candidate budget $Q$, pre-screening size $M_{\mathrm{pool}}$, number $N_{\mathrm{gen}}$ of generated trajectories, terminal-score multiplier $\lambda_{\mathrm{tar}}$, guidance scale $s_{\mathrm{cfg}}$, diffusion steps $T_{\mathrm{diff}}$\\
\textbf{Output}: Candidate set $\{\bm{x}^{\mathrm{out}}_1,\ldots,\bm{x}^{\mathrm{out}}_Q\}$
\begin{algorithmic}[1]
\STATE Initialize the candidate pool $\mathcal{P}\leftarrow\emptyset$;
\FOR{$u=1,\ldots,N_{\mathrm{gen}}$}
    \STATE Sample $\bm{X}^{(T_{\mathrm{diff}})}\sim\mathcal{N}(0,I)$ and $\tau^{\mathrm{ctx}}\sim D_{\mathrm{traj}}$;
    \STATE Set $\bm{X}^{(T_{\mathrm{diff}})}_{1:H_{\mathrm{ctx}}}\leftarrow[\bm{x}^{\tau^{\mathrm{ctx}}}_1,\ldots,\bm{x}^{\tau^{\mathrm{ctx}}}_{H_{\mathrm{ctx}}}]$;
    \FOR{$t=T_{\mathrm{diff}},\ldots,1$}
        \STATE Compute the guided noise using Eq.~(\refeq{eq:computeeps});
        \STATE Compute $\bm{X}^{(t-1)}$ using Eq.~(\refeq{eq:ddpm_reverse});
        \STATE Apply context conditioning using Eq.~(\refeq{eq:tauupdating})
    \ENDFOR
    \STATE $\mathcal{P} \leftarrow \mathcal{P} \cup \bm{X}^{(0)}_{H_{\mathrm{ctx+1}}:H}$ 
\ENDFOR
\STATE Retain the top $M_{\mathrm{pool}}$ valid designs in $\mathcal{P}$ according to $\mu_{\phi}$;
\STATE Partition them into $Q$ clusters $\{\mathcal{B}_1,\ldots,\mathcal{B}_Q\}$ using $k$-means;
\FOR{$q=1,\ldots,Q$}
    \STATE $\bm{x}^{\mathrm{out}}_q\leftarrow\operatorname*{arg\,max}_{\bm{x}\in\mathcal{B}_q}\mu_{\phi}(\bm{x})$
\ENDFOR
\RETURN $\{\bm{x}^{\mathrm{out}}_1,\ldots,\bm{x}^{\mathrm{out}}_Q\}$
\end{algorithmic}
\end{algorithm}

\paragraph{Comparison with GTG~\cite{yun2024guided}}
UGTL adapts GTG's conditional trajectory diffusion for trajectory modeling. However, GTG constructs trajectories by sampling from score-feasible nearest neighbors, which improves locality but does not distinguish these neighbors according to their alignment with an optimization direction. UGTL instead starts from our proposed algorithm-dependent learnability and uses uncertainty-aware proxy-gradient guidance to construct optimizer-like improvement trajectories (yielding more consistent trajectories, as shown in Figure~\ref{fig:branin_experiment}), followed by Cluster-then-Select for candidate selection.
Thus, compared with GTG, UGTL retains the trajectory-modeling pipeline while introducing different procedures for trajectory construction and candidate selection. 

\section{Experiments}
\label{sec:experiments}

In this section, we evaluate the performance of our proposed method UGTL empirically on popular benchmarks, Design-Bench~\citep{design-bench} and BBOB~\cite{bbob-functions}, for offline data-driven optimization. To assess UGTL systematically, this section organizes the evaluation into three complementary levels, moving from end-to-end effectiveness to trajectory-level behavior and then to component contributions and robustness. First, at the end-to-end level, we examine whether UGTL is competitive with existing offline optimization methods in terms of final candidate quality. Second, at the trajectory level, we assess whether the proposed trajectory constructor produces locally coherent, improvement-oriented paths and whether these properties are preserved in model-generated trajectories. Finally, at the component level, we test whether the UGTL constructor remains beneficial across different downstream models and training procedures, isolate the additional contribution of the diversity-aware Cluster-then-Select strategy, and assess sensitivity to the main hyperparameters.

\subsection{Experimental Settings}
\label{subsec:experimental-settings}
\paragraph{Benchmark and tasks}
We benchmark UGTL on five Design-Bench tasks~\cite{design-bench}, following recent offline data-driven optimization studies~\cite{ltr,yun2024guided}. The continuous tasks include: 1) Ant Morphology~\cite{gym}, which optimizes a 60-dimensional morphology for locomotion; 2) D'Kitty Morphology~\cite{robel}, which optimizes a 56-dimensional robot morphology; and 3) Superconductor~\cite{superconductor}, which searches an 86-dimensional material representation to maximize critical temperature. The discrete tasks are TF-Bind-8 and TF-Bind-10~\cite{tfbind}, which optimize DNA sequences of length 8 and 10 for transcription-factor binding affinity. 
For each task, we directly use the offline datasets provided by Design-Bench, which contain 10{,}004, 10{,}004, 17{,}014, 32{,}898, and 50{,}000 designs for Ant, D'Kitty, Superconductor, TF-Bind-8, and TF-Bind-10, respectively. 

\paragraph{Controlled BBOB case studies}
Detailed trajectory analysis requires evaluating every intermediate design. Since the evaluations for Design-Bench tasks are quite expensive, Design-Bench mainly measures final candidate quality. We introduce two cheap analytic BBOB functions~\cite{bbob-functions} as controlled case studies to examine the trajectories produced by different methods.
 Specifically, we use Rastrigin, a multimodal function, and Rosenbrock, a narrow-valley function, with dimensions $d\in\{5,10,15,20\}$. 
 For each function and dimension, we sample 50{,}000 points uniformly from the native domain ($[-5.12,5.12]^d$ for Rastrigin and $[-5,5]^d$ for Rosenbrock), transform the minimization objective into a score in $[0,1]$ using min--max scaling, and remove the top $40\%$ by score to form an offline dataset of 30{,}000 points.

\paragraph{Compared methods}
We compare UGTL with five categories of methods, including 24 baselines. The first class consists of baseline optimization methods that maximize a surrogate model, including BO-$q$EI~\cite{bo-book,bo-tutorial}, CMA-ES~\cite{cma-es}, REINFORCE~\cite{reinforce}, and Gradient Ascent with its mean- and minimum-ensemble variants. The second category comprises inverse or conditional generative methods, including CbAS~\cite{cbas}, MINs~\cite{mins}, and DDOM~\cite{ddom}. The third category is called regularized surrogate methods, including COMs~\cite{coms}, RoMA~\cite{roma}, IOM~\cite{iom}, BDI~\cite{bdi}, ICT~\cite{ict}, Tri-Mentoring~\cite{tri-mentoring}, FGM~\cite{fgm}, LTR~\cite{ltr}, GABO~\cite{gabo}, and DynAMO-Adam~\cite{dynamo}. We further compare trajectory-informed methods, MATCH-OPT~\cite{hoang2024learning} and ROOT~\cite{dao2025root}, and generative trajectory methods, BONET~\cite{mashkaria2023generative}, GTG~\cite{yun2024guided}, and PGS~\cite{chemingui2024offline}. Except for ROOT, LTR, GABO, and DynAMO-Adam, the baseline results are taken from~\cite{ltr} under the same Design-Bench protocol. We run ROOT, LTR, GABO, and the Adam variant of DynAMO using their official implementations.

\paragraph{Implementation details and hyperparameters}
For trajectory construction, we follow the hyperparameters in~\cite{yun2024guided} by setting the trajectory horizon to $H=64$, the neighborhood size to $K=20$, the initial percentile to $p_{\mathrm{init}}=20\%$, and the number $N_{\mathrm{traj}}$ of trajectories to 4{,}000 for continuous tasks and 1{,}000 for discrete tasks. The score tolerance $\xi=0.01$ on D'Kitty and $0.05$ on the other tasks. Each surrogate is an independently initialized two-hidden-layer MLP of width 1{,}024 with LeakyReLU activations trained for 5{,}000 Adam steps with learning rate $10^{-3}$ and batch size 128. We set the ensemble size $M_{\mathrm{ens}}=10$, retention ratio $\rho=0.5$, and uncertainty temperature $c_{\mathrm{temp}}=0.5$.

For trajectory modeling, the conditional diffusion model is a temporal U-Net~\citep{pmlr-v162-janner22a} with channel multipliers $(1,4,8)$, base width 128 on continuous tasks and 32 on discrete tasks, $T_{\mathrm{diff}}=200$ steps, and a cosine noise schedule. 
The diffusion model is trained by Adam~\cite{adam} using a batch size of 128 and EMA decay of 0.995. Models for continuous tasks are trained for 50{,}000 steps at a learning rate of $10^{-4}$, while those for discrete tasks are trained for 20{,}000 steps at $2\times10^{-4}$. The condition-drop probability $p_{\mathrm{drop}}=0.25$. We report the results of the final checkpoint. 

For candidate generation, the classifier-free guidance scale, context length, number of generated trajectories, and pre-screening pool size are $s_{\mathrm{cfg}}=1.5$, $H_{\mathrm{ctx}}=32$, $N_{\mathrm{gen}}=1000$, and $M_{\mathrm{pool}}=2048$, respectively. $\lambda_{\mathrm{tar}}$ is set to $1.5$ on TF-Bind-8 and $1.3$ otherwise. Cluster-then-Select applies $k$-means in the normalized diffusion representation, using $k$-means++ initialization~\citep{arthur2007kmeans} with 10 independent starts. 

Additionally, in the controlled BBOB comparison, we set $H=64$, $K=50$, $N_{\mathrm{traj}}=1000$, $p_{\mathrm{init}}=20\%$, $\xi=0.05$, and $N_{\mathrm{gen}}=1000$. All experiments are implemented in PyTorch 1.13.1.

\paragraph{Evaluation and metrics}
Each method returns $Q=128$ candidates. Following the Design-Bench protocol~\cite{design-bench,ltr}, we evaluate the candidates using the ground-truth oracle and report the normalized $100$th-percentile score, i.e., the score of the best returned design. A raw score $y$ is normalized as $(y-y_{\min}^{\mathrm{all}})/(y_{\max}^{\mathrm{all}}-y_{\min}^{\mathrm{all}})$, where $y_{\min}^{\mathrm{all}}$ and $y_{\max}^{\mathrm{all}}$ are the extrema of the full, unobserved benchmark dataset. We report the mean $\pm$ standard deviation over eight independent runs. 

\subsection{Main Results on Design-Bench}
\label{subsec:designbench}

Table~\ref{tab:design-bench-max-score} reports the $100$th-percentile normalized scores. Among the 25 methods, UGTL achieves the best mean rank of $3.1$, followed by LTR~\citep{ltr} with $4.8$ and BDI~\citep{bdi} with $5.9$. UGTL obtains the highest mean score on Superconductor and TF-Bind-10 and improves over the best observed offline score, $\mathcal{D}(\mathrm{best})$, on all five tasks. Among generative trajectory methods, UGTL achieves the highest mean score on every task compared with BONET~\citep{mashkaria2023generative}, GTG~\citep{yun2024guided}, and PGS~\citep{chemingui2024offline}. These results demonstrate the overall effectiveness of UGTL on Design-Bench and also motivate the trajectory-level analyses below.

\begin{table}[t!]
\centering
\caption{$100$th-percentile normalized score among $Q=128$ candidates on five Design-Bench tasks (mean $\pm$ standard deviation over eight runs). {\color{rank1blue}\textbf{Blue}} and {\color{rank2purple}\textbf{Violet}} denote the best and runner-up reported means. $\mathcal{D}(\mathrm{best})$ is the best score in the offline dataset. Mean rank is computed across all 25 methods with average ranks for ties.}
\renewcommand{\arraystretch}{1.1}
\setlength{\tabcolsep}{4pt}
\resizebox{\linewidth}{!}{
\begin{tabular}{cl|ccccc|c}
\toprule
\multicolumn{1}{c}{Category} & \multicolumn{1}{l|}{Method} & Ant & D'Kitty & Superconductor & TF-Bind-8 & TF-Bind-10 & Mean Rank \\
\midrule
\ & \multicolumn{1}{l|}{$\mathcal{D}(\text{best})$} & 0.565 & 0.884 & 0.400 & 0.439 & 0.467 & / \\
\midrule
\multirow{6}{*}{\makecell[c]{Proxy\\Optimization}}
 & BO-$q$EI
 & $0.812 \pm 0.000$
 & $0.896 \pm 0.000$
 & $0.382 \pm 0.013$
 & $0.802 \pm 0.081$
 & $0.628 \pm 0.036$
 & 19.4/25 \\

 & CMA-ES
 & $\mathbf{\color{rank1blue} 1.712 \pm 0.754}$
 & $0.725 \pm 0.002$
 & $0.463 \pm 0.042$
 & $0.944 \pm 0.017$
 & $0.641 \pm 0.036$
 & 12.2/25 \\

 & REINFORCE
 & $0.248 \pm 0.039$
 & $0.541 \pm 0.196$
 & $0.478 \pm 0.017$
 & $0.935 \pm 0.049$
 & $\mathbf{\color{rank2purple} 0.673 \pm 0.074}$
 & 15.0/25 \\

 & Grad.\ Ascent
 & $0.273 \pm 0.023$
 & $0.853 \pm 0.018$
 & $0.510 \pm 0.028$
 & $0.969 \pm 0.021$
 & $0.646 \pm 0.037$
 & 12.0/25 \\

 & Grad.\ Ascent Mean
 & $0.306 \pm 0.053$
 & $0.875 \pm 0.024$
 & $0.508 \pm 0.019$
 & $\mathbf{\color{rank1blue} 0.985 \pm 0.008}$
 & $0.633 \pm 0.030$
 & 11.8/25 \\

 & Grad.\ Ascent Min
 & $0.282 \pm 0.033$
 & $0.884 \pm 0.018$
 & $\mathbf{\color{rank2purple} 0.514 \pm 0.020}$
 & $\mathbf{\color{rank2purple} 0.979 \pm 0.014}$
 & $0.632 \pm 0.027$
 & 11.6/25 \\
\midrule
\multirow{3}{*}{\makecell[c]{Inverse/Conditional\\Generative Modeling}}
 & CbAS
 & $0.846 \pm 0.032$
 & $0.896 \pm 0.009$
 & $0.421 \pm 0.049$
 & $0.921 \pm 0.046$
 & $0.630 \pm 0.039$
 & 16.7/25 \\

 & MINs
 & $0.906 \pm 0.024$
 & $0.939 \pm 0.007$
 & $0.464 \pm 0.023$
 & $0.910 \pm 0.051$
 & $0.633 \pm 0.034$
 & 14.0/25 \\

 & DDOM
 & $0.908 \pm 0.024$
 & $0.930 \pm 0.005$
 & $0.452 \pm 0.028$
 & $0.913 \pm 0.047$
 & $0.616 \pm 0.018$
 & 15.8/25 \\
\midrule
\multirow{10}{*}{\makecell[c]{Regularized\\Surrogate Modeling}}
 & COMs
 & $0.916 \pm 0.026$
 & $0.949 \pm 0.016$
 & $0.460 \pm 0.040$
 & $0.953 \pm 0.038$
 & $0.644 \pm 0.052$
 & 10.3/25 \\

 & RoMA
 & $0.430 \pm 0.048$
 & $0.767 \pm 0.031$
 & $0.494 \pm 0.025$
 & $0.665 \pm 0.000$
 & $0.553 \pm 0.000$
 & 20.1/25 \\

 & IOM
 & $0.889 \pm 0.034$
 & $0.928 \pm 0.008$
 & $0.491 \pm 0.034$
 & $0.925 \pm 0.054$
 & $0.628 \pm 0.036$
 & 14.1/25 \\

 & BDI
 & $\mathbf{\color{rank2purple} 0.963 \pm 0.000}$
 & $0.941 \pm 0.000$
 & $0.508 \pm 0.013$
 & $0.973 \pm 0.000$
 & $0.658 \pm 0.000$
 & 5.9/25 \\

 & ICT
 & $0.915 \pm 0.024$
 & $0.947 \pm 0.009$
 & $0.494 \pm 0.026$
 & $0.897 \pm 0.050$
 & $0.659 \pm 0.024$
 & 10.0/25 \\

 & Tri-Mentoring
 & $0.891 \pm 0.011$
 & $0.947 \pm 0.005$
 & $0.503 \pm 0.013$
 & $0.956 \pm 0.000$
 & $0.662 \pm 0.012$
 & 8.0/25 \\

 & FGM
 & $0.923 \pm 0.023$
 & $0.944 \pm 0.014$
 & $0.481 \pm 0.024$
 & $0.811 \pm 0.079$
 & $0.611 \pm 0.008$
 & 14.2/25 \\

 & LTR
 & $0.920 \pm 0.026$
 & $\mathbf{\color{rank2purple} 0.958 \pm 0.011}$
 & $0.509 \pm 0.027$
 & $0.977 \pm 0.008$
 & $0.655 \pm 0.013$
 & $\mathbf{\color{rank2purple} 4.8/25}$ \\

 & GABO
 & $0.038 \pm 0.012$
 & $0.719 \pm 0.001$
 & $0.374 \pm 0.020$
 & $0.926 \pm 0.038$
 & $0.619 \pm 0.043$
 & 21.1/25 \\

 & DynAMO-Adam
 & $0.113 \pm 0.085$
 & $0.789 \pm 0.059$
 & $0.413 \pm 0.106$
 & $0.719 \pm 0.142$
 & $0.556 \pm 0.090$
 & 23.0/25 \\
\midrule
\multirow{2}{*}{\makecell[c]{Trajectory-Informed\\Optimization}}
 & MATCH-OPT
 & $0.933 \pm 0.016$
 & $0.952 \pm 0.008$
 & $0.504 \pm 0.021$
 & $0.824 \pm 0.067$
 & $0.655 \pm 0.050$
 & 8.7/25 \\

 & ROOT
 & $0.955 \pm 0.017$
 & $\mathbf{\color{rank1blue} 0.971 \pm 0.005}$
 & $0.451 \pm 0.032$
 & $0.977 \pm 0.015$
 & $0.653 \pm 0.030$
 & 7.1/25 \\
\midrule
\multirow{4}{*}{\makecell[c]{Generative\\Trajectory Learning}}
 & BONET
 & $0.921 \pm 0.031$
 & $0.949 \pm 0.016$
 & $0.390 \pm 0.022$
 & $0.798 \pm 0.123$
 & $0.575 \pm 0.039$
 & 16.5/25 \\

 & GTG
 & $0.855 \pm 0.044$
 & $0.942 \pm 0.017$
 & $0.480 \pm 0.055$
 & $0.910 \pm 0.040$
 & $0.619 \pm 0.029$
 & 15.0/25 \\

 & PGS
 & $0.715 \pm 0.046$
 & $0.954 \pm 0.022$
 & $0.444 \pm 0.020$
 & $0.889 \pm 0.061$
 & $0.634 \pm 0.040$
 & 14.6/25 \\

 & \textbf{UGTL}
 & $0.958 \pm 0.012$
 & $0.955 \pm 0.013$
 & $\mathbf{\color{rank1blue} 0.535 \pm 0.030}$
 & $0.956 \pm 0.013$
 & $\mathbf{\color{rank1blue} 0.674 \pm 0.044}$
 & $\mathbf{\color{rank1blue} 3.1/25}$ \\
\bottomrule
\end{tabular}
}
\label{tab:design-bench-max-score}
\end{table}

\subsection{Trajectory-quality Analysis}
\label{subsec:traj-synthesis-analysis}

The comparison in Section~\ref{subsec:designbench} measures final candidate quality but does not show how complete trajectories evolve. We therefore complement it with trajectory-level diagnostics and the controlled BBOB case studies described in Section~\ref{subsec:experimental-settings}. The analytically evaluable objectives of BBOB allow us to evaluate every point in both constructed and generated trajectories while keeping the downstream diffusion model and sampler fixed.

\subsubsection{Trajectory-quality Diagnostics}
In this subsection, we compare the trajectory quality of BONET~\citep{mashkaria2023generative}, PGS~\citep{chemingui2024offline}, GTG~\citep{yun2024guided}, and UGTL. On Design-Bench, we evaluate the constructed trajectories; on BBOB, we evaluate both constructed and model-generated trajectories. For the generated BBOB trajectories, all methods use the same conditional diffusion model, training schedule, and full-trajectory sampler.
We evaluate the final sampled trajectories.
For a trajectory $\tau$, we measure cumulative regret and smoothness as
\begin{equation}
    R(\tau)=\sum_{k=1}^{H}(1-s_k),
    \qquad
    S(\tau)=\frac{1}{H-2}\sum_{k=1}^{H-2}
    \|\bm{x}_{k+2}-2\bm{x}_{k+1}+\bm{x}_k\|^2,
\end{equation}
where $s_k$ is the normalized objective value. The regret is small when a trajectory remains in high-value regions. Since $\bm{x}_{k+2}-2\bm{x}_{k+1}+\bm{x}_k=(\bm{x}_{k+2}-\bm{x}_{k+1})-(\bm{x}_{k+1}-\bm{x}_k)$ is the difference between two consecutive displacement vectors, a smaller smoothness value indicates less variation between successive moves and hence a smoother trajectory. Before computing smoothness, we normalize the design coordinates using min--max scaling. 
We report Design-Bench smoothness only on the three continuous tasks because distances between discrete latent representations are not directly comparable.
Each diagnostic is computed over eight independent runs.

Figure~\ref{fig:trajectory-quality-summary} summarizes the trajectory diagnostics. UGTL ranks first on Design-Bench regret ($1.2/4$) and smoothness ($1.3/4$). On BBOB, it ranks first on constructed-trajectory regret and on both regret and smoothness for generated trajectories, while tying GTG on constructed-trajectory smoothness. These results show that UGTL produces trajectories that remain in high-value regions and change more smoothly between successive moves.

\begin{figure}[t!]
    \centering
    \includegraphics[width=0.98\linewidth]{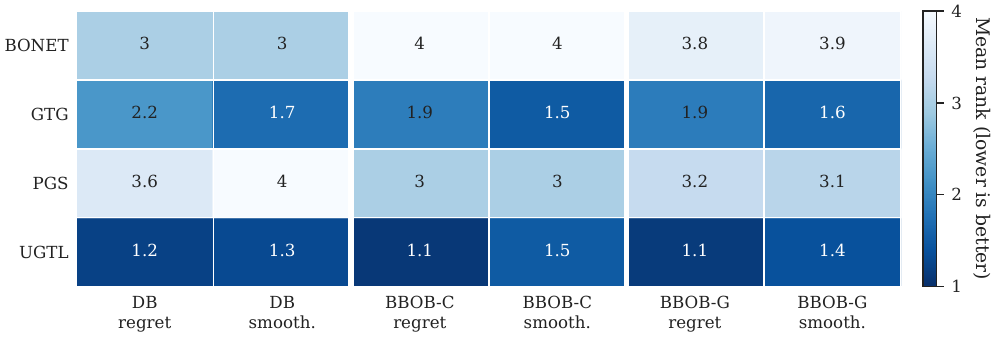}
    \caption{Mean-rank summary of trajectory quality (lower is better). DB reports constructed-trajectory regret over five Design-Bench tasks and smoothness over its three continuous tasks. BBOB reports regret/smoothness for constructed (C) and generated (G) trajectories over eight function--dimension pairs. Each cell is the task-averaged rank of one method.}
    \label{fig:trajectory-quality-summary}
\end{figure}

\subsubsection{Controlled BBOB Comparison}
\label{subsec:bbob}

Figure~\ref{fig:bbob-performance} reports the controlled BBOB results, i.e., the 100th-percentile normalized scores achieved by different trajectory construction methods (using the same downstream diffusion model and sampler) in BBOB tasks. UGTL obtains the highest mean on all four Rastrigin tasks and three of the four Rosenbrock tasks. It achieves the best mean rank of $1.1/4$, compared with $2.3/4$ for GTG, $3.0/4$ for PGS, and $3.6/4$ for BONET. Since the downstream diffusion model and sampler are shared, the improvement directly demonstrates the effect of the trajectory constructor in this controlled setting.

\begin{figure}[H]
    \centering
    \includegraphics[width=0.98\linewidth]{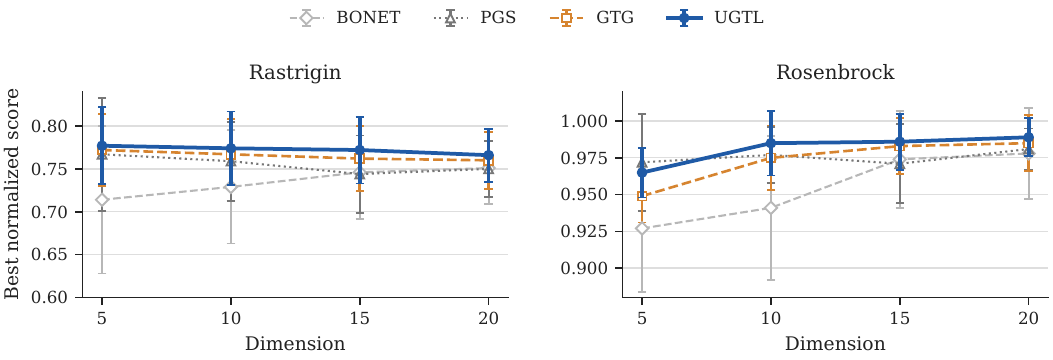}
    \caption{Controlled BBOB performance as the dimension increases (mean $\pm$ standard deviation over eight runs). Each result is the final 100th percentile normalized score achieved by different trajectory construction methods, with standard deviation across different runs. All methods use the same downstream diffusion model and sampler.}
    \label{fig:bbob-performance}
\end{figure}

\subsection{Ablation Studies}
\label{subsec:ablation}

We first examine whether the UGTL constructor can transfer to other trajectory models and whether Cluster-then-Select improves candidate generation. Then, we study sensitivity to the main hyperparameters of trajectory construction and candidate generation.

\subsubsection{Trajectory Construction and Candidate Selection}
\label{subsubsec:ablation-traj}
\label{subsubsec:ablation-cluster}

\paragraph{Trajectory-constructor replacement}
We replace the original constructor in BONET's Transformer pipeline, PGS's offline-RL pipeline, and GTG's diffusion pipeline with the UGTL constructor, while leaving their downstream models and training procedures unchanged. Figure~\ref{fig:module-ablation-summary}(a) reports the change from each original pipeline. The replacement improves 13 of the 15 architecture--task pairs and increases the five-task average for all three architectures, with the largest improvement in the offline-RL pipeline. This transfer shows that the constructed trajectories are useful beyond UGTL's diffusion model.

\paragraph{Candidate selection}
Figure~\ref{fig:module-ablation-summary}(b) compares Cluster-then-Select with directly returning the top $Q$ candidates ranked by the proxy. Cluster-then-Select improves the mean score on all five tasks, showing the benefit of distributing the evaluation budget across different high-proxy regions.

\begin{figure}[h!]
    \centering
    \includegraphics[width=0.98\linewidth]{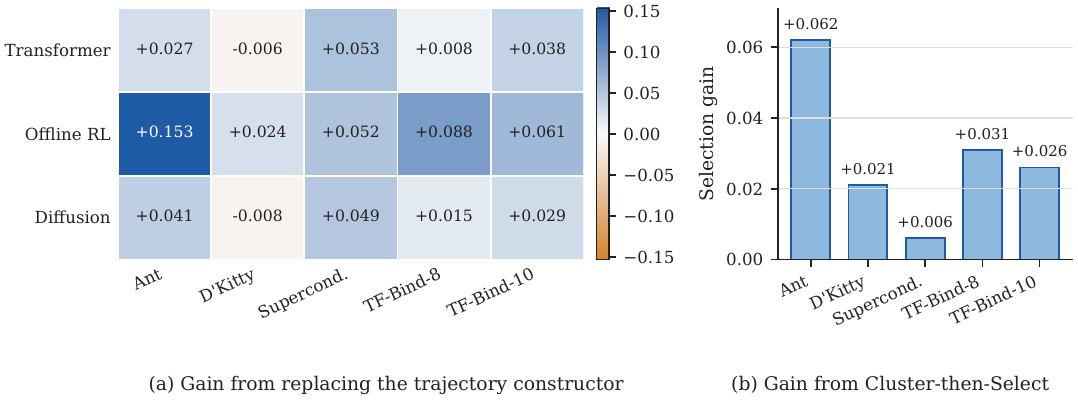}
    \caption{Module ablations on Design-Bench. (a) Change in normalized score after replacing the native constructor of each downstream architecture with the UGTL constructor. (b) Change from applying Cluster-then-Select to the diffusion pipeline. Positive values indicate improvement.}
    \label{fig:module-ablation-summary}
\end{figure}

\subsubsection{Hyperparameter Sensitivity}
\label{subsubsec:ablation-hyperparam}

To better understand the sensitivity of UGTL, we vary one hyperparameter at a time while fixing the others at their main values. The sweeps use $N_{\mathrm{traj}}\in\{1000,2000,4000\}$, $H\in\{32,64,128\}$, $K\in\{10,20,50\}$, $\xi\in\{0.01,0.05,0.10\}$, $p_{\mathrm{init}}\in\{10\%,20\%,30\%\}$, and $\lambda_{\mathrm{tar}}\in\{1.0,1.3,1.5\}$. The three $M_{\mathrm{pool}}$ settings shown as 1k, 2k, and 4k are 1024, 2048 and 4096, respectively. Figures~\ref{fig:sensitivity-nraw}--\ref{fig:sensitivity-target-multiplier} report the change in normalized score from each task's main setting. Black rings indicate the settings used in the main experiments, while horizontal error bars show $\pm$ one standard deviation over eight independent runs.

\begin{figure}[h!]
    \centering
    \includegraphics[width=0.88\linewidth]{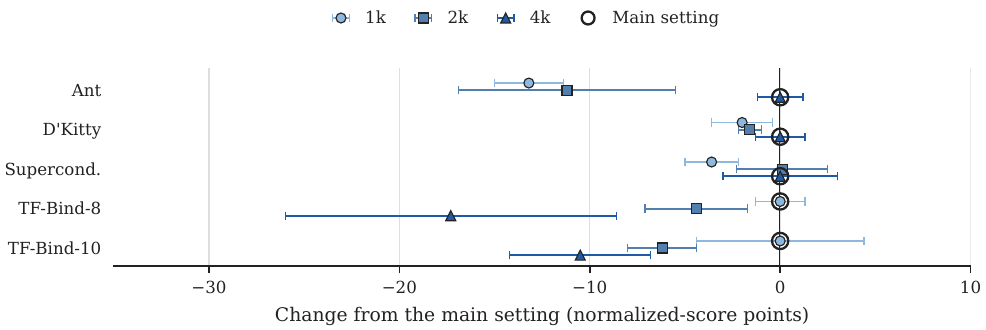}
    \caption{Sensitivity to the number $N_{\mathrm{traj}}$ of constructed trajectories. Points show the change from each task's main setting; black rings identify the main settings. Higher is better.}
    \label{fig:sensitivity-nraw}
\end{figure}

\begin{figure}[h!]
    \centering
    \includegraphics[width=0.88\linewidth]{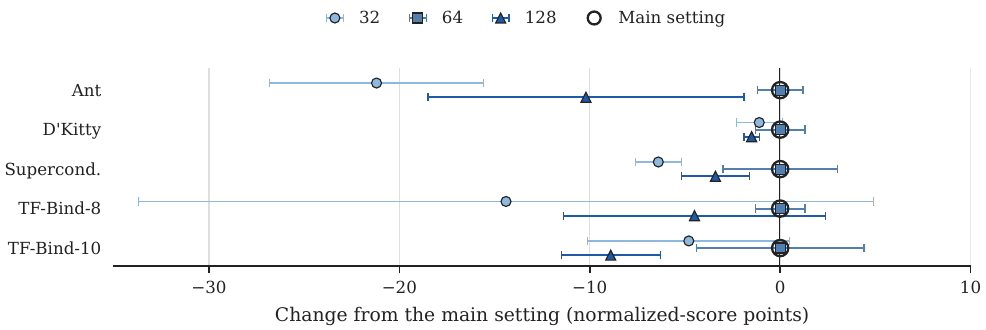}
    \caption{Sensitivity to the trajectory horizon $H$. Points show the change from each task's main setting; black rings identify the main setting. Higher is better.}
    \label{fig:sensitivity-horizon}
\end{figure}

\begin{figure}[h!]
    \centering
    \includegraphics[width=0.88\linewidth]{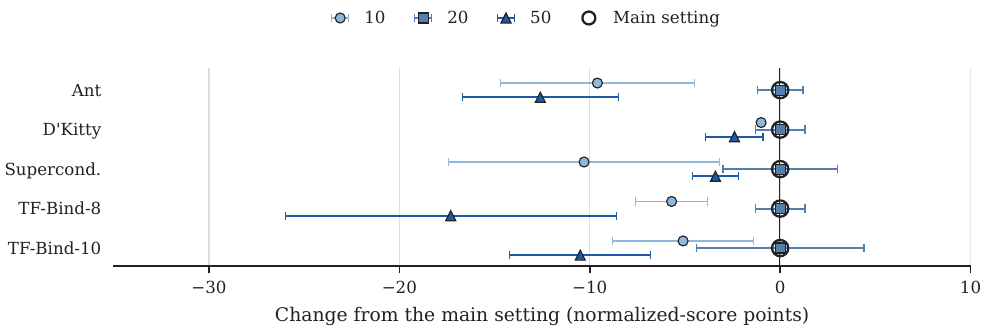}
    \caption{Sensitivity to the neighborhood size $K$. Points show the change from each task's main setting; black rings identify the main setting. Higher is better.}
    \label{fig:sensitivity-neighbors}
\end{figure}

\begin{figure}[t!]
    \centering
    \includegraphics[width=0.88\linewidth]{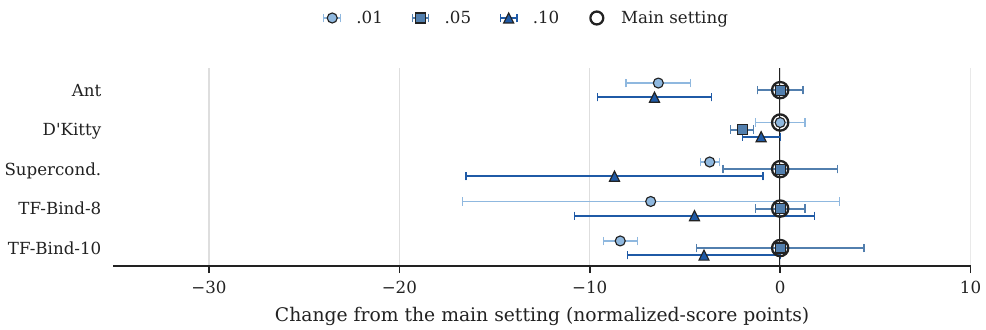}
    \caption{Sensitivity to the score-relaxation tolerance $\xi$. Points show the change from each task's main setting; black rings identify the main settings. Higher is better.}
    \label{fig:sensitivity-tolerance}
\end{figure}

\begin{figure}[h!]
    \centering
    \includegraphics[width=0.88\linewidth]{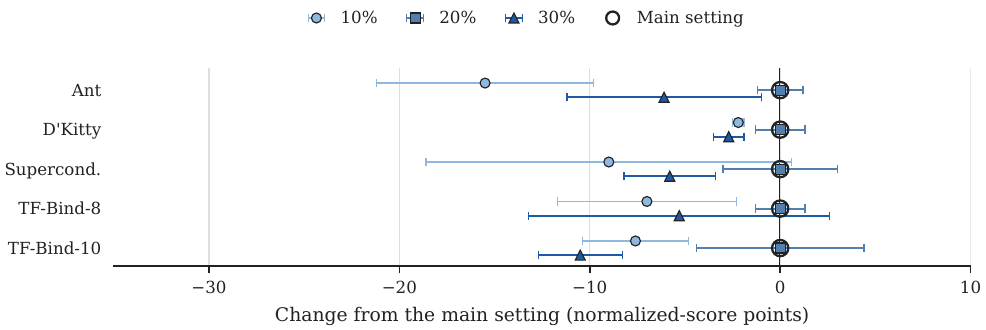}
    \caption{Sensitivity to the initial percentile $p_{\mathrm{init}}$. Points show the change from each task's main setting; black rings identify the main setting. Higher is better.}
    \label{fig:sensitivity-initial-percentile}
\end{figure}

\begin{figure}[h!]
    \centering
    \includegraphics[width=0.88\linewidth]{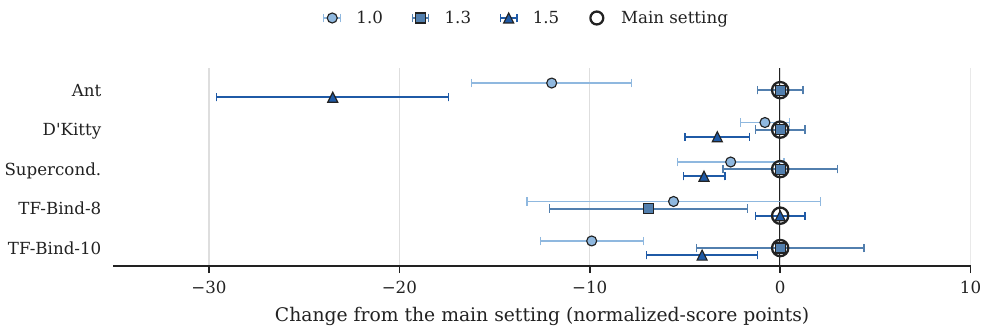}
    \caption{Sensitivity to the target multiplier $\lambda_{\mathrm{tar}}$. Points show the change from each task's main setting; black rings identify the main settings. Higher is better.}
    \label{fig:sensitivity-target-multiplier}
\end{figure}

\begin{figure}[h!]
    \centering
    \includegraphics[width=0.88\linewidth]{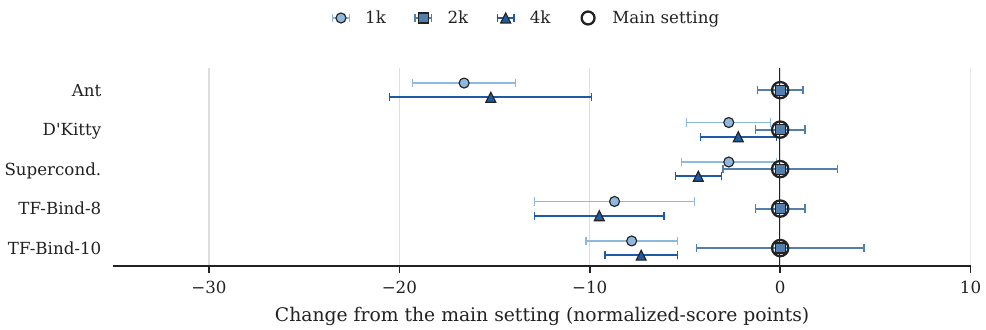}
    \caption{Sensitivity to the pre-screening pool size $M_{\mathrm{pool}}$. Points show the change from each task's main setting; black rings identify the main setting. Higher is better.}
    \label{fig:sensitivity-pool-size}
\end{figure}

We can observe from Figures~\ref{fig:sensitivity-nraw}--\ref{fig:sensitivity-target-multiplier} consistent patterns across tasks. The common settings $H=64$, $K=20$, $p_{\mathrm{init}}=20\%$, and $M_{\mathrm{pool}}=2048$ obtain the best mean on all five tasks in their respective studies, showing that these hyperparameters can be shared across tasks. The remaining parameters also follow simple choices: for trajectory construction, the three continuous tasks use $N_{\mathrm{traj}}=4{,}000$ and the two discrete tasks use $N_{\mathrm{traj}}=1{,}000$, while $\xi=0.05$ and $\lambda_{\mathrm{tar}}=1.3$ work well on four of the five tasks. Overall, UGTL admits a unified configuration for most hyperparameters, with only limited adjustment for the remaining ones, suggesting that extensive per-task tuning is unnecessary.

\section{Conclusion}\label{sec:conclusion}

This paper revisits learnability in offline data-driven optimization from an algorithm-dependent perspective. PAC or PMAC learnability can leave the optimal region unreliable even when average prediction under the sampling distribution is accurate, and thus cannot guarantee good approximations for offline optimization~\cite{balkanski2017sample,balkanski2022limitations,balkanski2017minimizing}. Algorithm-dependent learnability instead localizes accuracy to the information queried along an optimizer's trajectory. Its value-query form yields good approximation guarantees for greedy and local search algorithms on representative submodular optimization problems, while its first-order form yields a good guarantee for projected gradient descent on convex minimization.

Motivated by algorithm-dependent learnability, we develop a trajectory-learning framework (which consists of trajectory construction, trajectory modeling, and candidate generation) and analyze existing trajectory-based methods in a unified way. Using this framework, we further propose UGTL, an uncertainty-aware gradient-guided instantiation with conditional diffusion and diversity-aware selection. To better reflect the behavior of optimizers, the trajectories are constructed with the guidance of the gradient and prediction uncertainty of an ensemble of surrogate models. UGTL achieves the best mean rank, $3.1/25$, among 25 methods on five Design-Bench tasks. Under matched downstream models and samplers, controlled BBOB experiments and cross-architecture replacements show that the constructor substantially shapes the improvement behavior learned by the downstream model; diversity-aware selection further raises the reported mean on all five tasks. Extending the principle of algorithm-dependent learnability to more complex offline scenarios (e.g., multi-objective~\cite{xue2024offline} and universal~\cite{tan2025towards} offline optimization) will be interesting future work.

\section{Acknowledgements}

We want to thank Dr. Shen-Huan Lyu from Hohai University for his helpful discussion and proofreading the paper. This work was supported by the National Science and Technology Major Project (2022ZD0116600), the National Science Foundation of China (62276124, 624B2069), the Fundamental Research Funds for the Central Universities (14380020), the Fundamental and Interdisciplinary Disciplines Breakthrough Plan of the Ministry of Education of China (JYB2025XDXM118), and the “111 Center” (B26023). Chao Qian is the corresponding author. 

\bibliographystyle{plain}
\bibliography{Rethink-Learnability-OfflineOpt}

\end{document}